\documentclass[10pt,conference]{IEEEtran}
\IEEEoverridecommandlockouts
\usepackage{algorithm}
\usepackage[fleqn]{amsmath}
\usepackage{amssymb,amsfonts}
\usepackage{cite}
\usepackage{comment}
\usepackage{framed}
\usepackage{graphicx}
\usepackage{mathalpha}
\usepackage{multirow}
\usepackage{textcomp}
\usepackage{xcolor}
\usepackage{subcaption}
\usepackage[T1]{fontenc}

\DeclareMathOperator*{\sdef}{\stackrel{\mathsf{def}}{=}}
\DeclareMathOperator*{\smax}{\mathsf{max}}

\definecolor{myblue}{RGB}{0, 0, 220}
\usepackage[colorlinks=true]{hyperref}
\hypersetup{
    linkcolor=myblue,          % Color of internal links (sections, figures, etc.)
    citecolor=myblue,        % Color of citation links
    urlcolor=myblue,          % Color of URL links
    filecolor=myblue       % Color of file links
}

\usepackage{amsthm}
\newtheoremstyle{ieeebold}  % Name of the style
  {\topsep}                  % Space above
  {\topsep}                  % Space below
  {\normalfont}              % Body font (normal, not italic)
  {0pt}                      % Indent amount (0pt = noindent)
  {\bfseries}                % Theorem head font (bold)
  {.}                        % Punctuation after head
  {5pt plus 1pt minus 1pt}   % Space after head
  {}                         % Custom head spec
\theoremstyle{ieeebold}

\newtheorem{lemma}{Lemma}

\newtheorem{proposition}{Proposition}
\newtheorem{theorem}{Theorem}

\newcounter{mycounter}
\newcommand{\myitem}{%
    \refstepcounter{mycounter}%
    (par\arabic{mycounter})%
}

\newcounter{challengecounter}
\newcounter{solutioncounter}
\newcounter{algcpmccounter}
\newcounter{ctr}
\DeclareRobustCommand{\stp}[1]{%
   \refstepcounter{ctr}%
   \thectr\label{#1}.}

\graphicspath{{figures/}}

\begin{document}

\title{Toward AI Trustworthiness: Finding Analytically Proven Forward-Invariant Sets for AI-Controlled Systems}

\author{
Haoyang Song$^{1}$,
Qixin Wang$^{2}$,
XikunYang$^{2}$,\\
Department of Computing, The Hong Kong Polytechnic University, Hong Kong SAR, China. \\
Email: haoyang.song@connect.polyu.hk \\
csqwang@polyu.edu.hk \\
frank-xikun.yang@connect.polyu.hk \\
% % anonymous authors
% \IEEEcompsocitemizethanks{
% \IEEEcompsocthanksitem
% 1. These authors contributed equally to this work.
% \IEEEcompsocthanksitem
% 2. This is the corresponding author.
% }
}

\maketitle

\begin{abstract}
The rise of Neural Network (NN) based AI fosters increasing demands to apply AI in control applications, particularly, to generate AI-controllers for nonlinear control applications, where traditional control theories often fail to work well. However, such AI-controlled systems are well-known to be highly nonlinear, and hard to explain/verify, leading to trustworthiness concerns. On the other hand, many AI-controlled systems are safety/mission critical, it often needs to be certifiably trustworthy. But due to the the aforementioned trustworthiness concerns, certification legal/regulatory authorities are deterred from approving the AI-controlled systems. 
A key solution to address the trustworthiness concerns is to find a Forward-Invariant Set (FIS) for the AI-controlled systems. A FIS is a subset of the state space of the controlled target, so that starting from any state inside of the FIS, the future state trajectory remains inside the FIS. If the FIS does not intersect with any concerned unsafe state set, the legal/regulatory authorities can be assured of the trustworthiness of the AI-controlled system. 
However, finding the FIS for a given (e.g. by the control system vendor to the authority for certification) AI-controlled system is a well-known hard problem. In this paper, we propose a solution to find analytically proven FISs for given AI-controlled systems. We propose to use an Invertible NN (INN) to transform the original state space into a latent state space, to reduce the high nonlinearity of the original AI-controlled system, so as to increase the likelihood of a regular-shaped FIS. We further propose a procedure to stimulate the INN to converge to a preferred hyper-rectangular FIS candidate and conduct verification. We proved that, when this procedure returns success, the hyper-rectangular FIS candidate (and its INN inversely transformed counterpart in the original state space) is an analytically proven FIS. 
We conducted extensive trials on representative AI-controlled system testbeds. Our solution succeeded in finding FISs for all the testbeds. Meanwhile, to our best knowledge, there is no other work in the literature that addresses exactly the same problem as ours. The best we can compare with is a baseline solution adapted from a state-of-the-art solution for a related (but not the same) problem. The baseline solution failed to find any FIS for our testbeds. Furthermore, for 40 of the 45 testbeds, our solution's time costs are significantly smaller than those of the baseline solution's. The centers of our found FISs also roughly meet the preferences set by domain experts.
\end{abstract}

%\begin{IEEEkeywords}
%AI, Trustworthiness.
%\end{IEEEkeywords}

%=============================
\section{Introduction}
\label{sect:Introduction}

The rise of \emph{Neural Network (NN) based AI} (simplified as ``\emph{AI}'' in the following)~\cite{lecun2015deep}\cite{mnih2015human}\cite{sutton1998reinforcement} fosters increasing demands to apply AI in control applications. Such demands are particularly valid for nonlinear control applications, where traditional control theories often fail to work well~\cite{khalil_nonlinear_2002}\cite{slotine1991applied}\cite{recht2019tour}\cite{kober2013reinforcement}. $\hfill$\myitem\label{myitem:ai_is_nn_ai}

For such control applications, AI does make significant progresses: various novel \emph{AI generated controllers} (simplified as ``\emph{AI-controllers}'' in the following) are reported~\cite{kober2013reinforcement}~\cite{hwangbo2019learning}\cite{lillicrap2019continuouscontroldeepreinforcement}\cite{radosavovic2024real}. However, adoption of AI-controllers is still hampered by their trustworthiness. 

On one hand, AI generated models, including AI-controllers, are well-known to be highly nonlinear, and hard to explain/verify. Particularly, the difficulty often worsens for those nonlinear control applications, where traditional control theories cannot work well (hence need AI-controllers). 
On the other hand, many nonlinear control applications are safety or mission critical, such as aviation, smart driving, surgical robots etc. Such applications often demand \emph{certifiably} trustworthy controllers due to legal or regulatory requirements~\cite{knight2002safety}\cite{ames2016control}\cite{paden2016survey}. 
If an AI-controller cannot be explained/verified, authorities would not dare to issue the necessary legal/regulatory certificates to approve the adoption of the AI-controller. In other words, the authorities cannot trust such AI-controllers. \phantom{x} $\hfill$\myitem\label{myitem:ai_controlled_system_is_highly_nonlinear}

To address the above trustworthiness problem, an AI-controlled system should better guarantee proven safety.

To guarantee proven safety, one approach is to find a so-called \emph{Forward-Invariant Set} (FIS) for the AI-controlled system, where the FIS does not intersect with the \emph{set of concerned unsafe states}. These concepts are explained as follows.

First, any control application has its control target, aka \emph{plant}. For example, a robotic control system's plant is the physical robot. The \emph{state of the plant} (simplified as ``\emph{state}'' in the following) at time instance $t$ can be denoted as a vector $\vec{x}(t) \in \mathbb{R}^{n}$, where $\mathbb{R}^{n}$ is called the $n$-dimensional \emph{state space}, and each dimension of $\vec{x}(t)$ is also called a \emph{component} of the state.  $\hfill$\myitem\label{myitem:state}

Second, a FIS $\mathcal{I}$ is a subset of $\mathbb{R}^{n}$, where given any \emph{initial time instance} $t_{0}$, then for any \emph{initial state} $\vec{x}(t_0) \in \mathcal{I}$, the future \emph{state trajectory} $\{\vec{x}(t)\}_{t \in [t_{0}, +\infty)}$ always remains in $\mathcal{I}$.  $\hfill$\myitem\label{myitem:definition_fis}

Third, the set $\mathcal{F}$ of concerned unsafe states is another subset of $\mathbb{R}^{n}$, whose elements (aka \emph{unsafe states}) are states considered unsafe by the legal/regulatory authority. Often, we also call $\mathcal{F}$ the \emph{forbidden set}; and call  $\bar{\mathcal{F}} \sdef (\mathbb{R}^{n} - \mathcal{F})$ the \emph{safe set}. $\hfill$\myitem\label{myitem:definition_set_of_concerned_unsafe_states}

If a FIS $\mathcal{I}$ is found, and $\mathcal{I} \cap \mathcal{F} = \varnothing$ is proven, then the AI-controlled system guarantees proven safety. Specifically, any state trajectory $\{\vec{x}(t)\}_{t \in [t_{0}, +\infty)}$ starting from an initial state $\vec{x}(t_{0}) \in \mathcal{I}$ is guaranteed to be safe: as $\{\vec{x}(t)\}_{t \in [t_{0}, +\infty)} \subseteq \mathcal{I}$ (due to the definition of FIS, see (par\ref{myitem:definition_fis})) implies $\{\vec{x}(t)\}_{t \in [t_{0}, +\infty)} \cap \mathcal{F} = \varnothing$ (due to $\mathcal{I} \cap \mathcal{F} = \varnothing$). The legal/regulatory authority can thus trust the AI-controlled system, given the initial state $\vec{x}(t_{0}) \in \mathcal{I}$.  $\hfill$\myitem\label{myitem:how_to_let_authority_trust}

However, how to find the FIS $\mathcal{I}$ is a well-known difficult problem~\cite{yang2024icml}\cite{Sidrane2022}, especially when the dynamics of the system involve high nonlinearity, which is a feature of AI-controlled systems (see (par\ref{myitem:ai_controlled_system_is_highly_nonlinear})). To our best knowledge, the state-of-the-art solution is by Yang et al.~\cite{yang2024icml}. Specifically, Yang et al.~\cite{yang2024icml} propose to jointly train the AI-controller and a Lyapunov function\footnote{Lyapunov function intuitively is a pseudo potential energy defined upon the states, with certain properties, such as always decreasing along any state trajectories. Traditional control theories often turn a controller design problem into a Lyapunov function construction problem~\cite{khalil_nonlinear_2002}. If the controller can be adjusted, so that a Lyapunov function can be constructed for the controlled system for a given initial state set, then the given initial state set is a FIS; in other words, the controller has a FIS, hence is stable. However, when the system is highly nonlinear, there is no generic solution to adjust the controller and construct the Lyapunov function.}, so that if the training succeeds, the outputted Lyapunov function also implies a FIS. However, in this method, the AI-controller is supposed to be jointly adjusted with the Lyapunov function (and the implied FIS) throughout the training. The method is not intended for problems where the AI-controller is given a priori (e.g., provided by a control
system vendor to the legal/regulatory authority for certification). Such problems, however, are the focus of this paper.  

In summary, \emph{the focus of this paper} is to find a FIS for an AI-controlled system $S$, where the AI-controller is given a priori. We propose the following solution heuristics. $\hfill$\myitem\label{myitem:focus_of_this_paper}

First, we notice a state-of-the-art verification tool, $\alpha,\beta$-CROWN~\cite{zhang2018nips}\cite{wang2021betanips}\cite{xu2020fast}\cite{xu2020nips}, which can \emph{analytically prove} if a given \emph{polytope} (see \eqref{eqn:polytope}) $\mathcal{B}$ actually bounds an NN's outputs on a given \emph{hyper-rectangular} (see \eqref{eqn:hyper-rectangle}) input domain of concern. Consider $S$ holistically as an NN, whose input is the original state $\vec{x}(t)$, and the output is the next state $\vec{x}(t+\delta)$ (where $\delta$ is the sensing/actuation period, see Section~\ref{sect:ProblemFormulation} for more discussions on discretization), i.e. both the input domain and output range are subsets of the $n$-dimensional state space. Then $S$ can be analyzed by $\alpha,\beta$-CROWN. 
Specifically, given a hyper-rectangular input domain of concern $\mathcal{X} \subseteq \mathbb{R}^{n}$, denote the corresponding output range of $S$ as $S(\mathcal{X}) \subseteq \mathbb{R}^{n}$. Let $\mathcal{B} = \mathcal{X}$ (note hyper-rectangle is polytope). Then if $\alpha,\beta$-CROWN can analytically prove that $S(\mathcal{X}) \subseteq \mathcal{B} = \mathcal{X}$, then $\mathcal{X}$ is intuitively a FIS. That is, $\alpha,\beta$-CROWN can help us analytically prove whether a hyper-rectangular FIS candidate $\mathcal{X}$ is actually a FIS for $S$.  

However, $\alpha,\beta$-CROWN requires $\mathcal{X}$ to be hyper-rectangular. The chance that a highly nonlinear system $S$'s FIS has a hyper-rectangular shape is low. To cope with this, we propose to construct a transformation function $\mathcal{T}$ and its inverse $\mathcal{T}^{-1}$, to transform the \emph{original} $n$-dimensional state space $\mathbb{X}_{0}$ to another (aka the \emph{latent}) $n$-dimensional state space  $\mathbb{X}_{1}$, so that $S$'s nonlinear behaviors in $\mathbb{X}_{0}$ become more \emph{regular} (to be formally defined in (par\ref{myitem:definition_regular})) in $\mathbb{X}_{1}$. We propose to use the tool of \emph{Invertible NN} (INN)~\cite{Laurent2017iclr} to find such $\mathcal{T}$ and $\mathcal{T}^{-1}$. $\hfill$\myitem\label{myitem:heuristics_use_inn}

Specifically, in $\mathbb{X}_{1}$, we shall specify an intended hyper-rectangle $\mathcal{I}_{1}$ to be the \emph{FIS candidate}. We train the INN of $\mathcal{T}$ and $\mathcal{T}^{-1}$, so that $\mathcal{I}_{1}$ looks like a FIS for the transformed $S$ in $\mathbb{X}_{1}$. Once the training succeeds, we shall exploit $\alpha,\beta$-CROWN to analytically prove if $\mathcal{I}_{1}$ is actually a FIS for the transformed $S$ in $\mathbb{X}_{1}$. If so, $\mathcal{I}_{0} = \mathcal{T}^{-1}(\mathcal{I}_{1})$ (i.e. inverse transform $\mathcal{I}_{1}$ from $\mathbb{X}_{1}$ to $\mathbb{X}_{0}$, see \eqref{eqn:definition_transformed_set}) shall be the FIS for $S$ in $\mathbb{X}_{0}$. Note we do not adjust $\mathcal{I}_{1}$, instead, we adjust the INN, as the latter can automatically adjust itself using backpropagation from a well designed loss function. $\hfill$\myitem\label{myitem:heuristics_fix_fis_adjust_inn}

The above heuristics are illustrated by Fig.~\ref{fig:overall_framework}. 

\begin{figure}[t]
\centering
\includegraphics[width=0.8\linewidth]{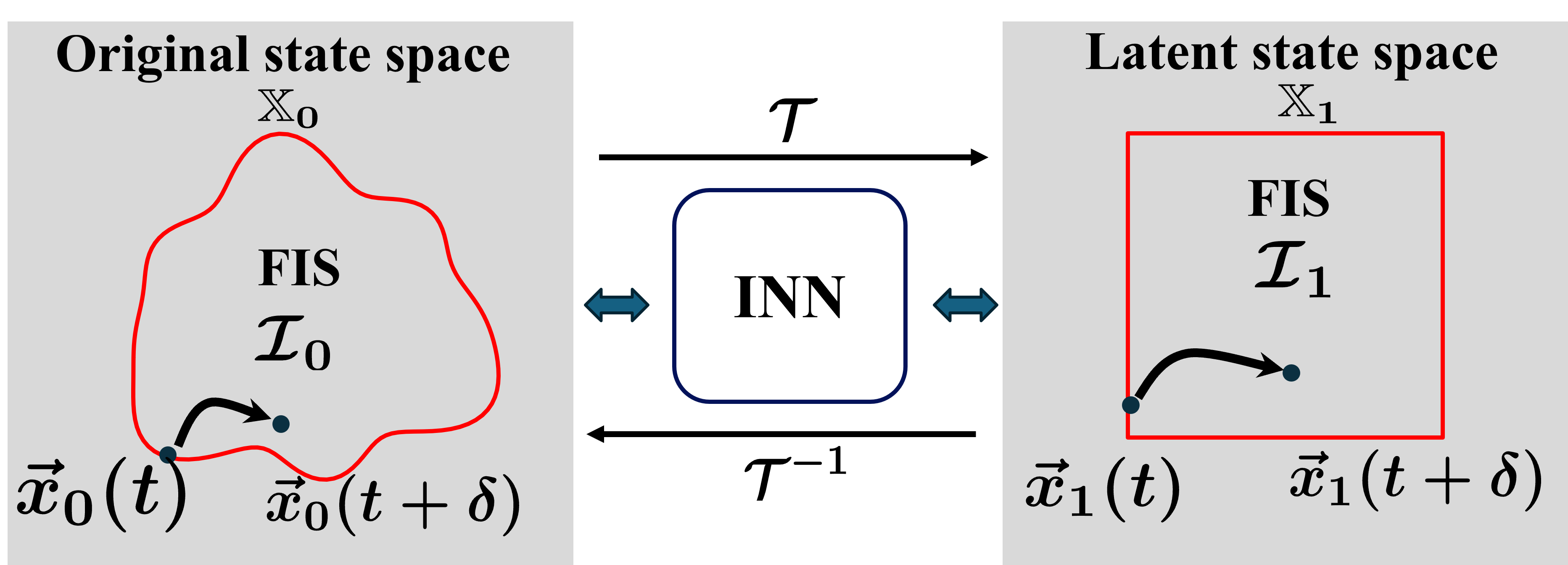}
\caption{Solution heuristics: a hyper-rectangular FIS constructed and proven in the latent state space maps to a nonlinear FIS in the original state space.}
\label{fig:overall_framework}
\end{figure}

In summary, this paper makes the following contributions.

\noindent 1) We propose a solution to find a FIS for an AI-controlled system $S$, where the AI-controller is given a priori. Specifically, our solution exploits INN to transform the complex and highly nonlinear $S$ from the original state space $\mathbb{X}_{0}$ to a latent state space $\mathbb{X}_{1}$, so that the transformed $S$ becomes more regular. We propose a specially designed loss function and training set, to train the INN, so that a hyper-rectangular FIS candidate $\mathcal{I}_{1}$ becomes more likely a real FIS for the transformed $S$ in $\mathbb{X}_{1}$.

\noindent 2) We theoretically prove that if $\mathcal{I}_{1}$ in $\mathbb{X}_{1}$ can pass a certain test by $\alpha,\beta$-CROWN, then it is an analytically proven hyper-rectangular FIS for $S$ in $\mathbb{X}_{1}$, and its INN inverse transformed counterpart $\mathcal{I}_{0}$ in $\mathbb{X}_{0}$ is an analytically proven FIS for $S$ in $\mathbb{X}_{0}$. 

\noindent 3) We carried out extensive evaluations on various AI-controlled systems. Our proposed solution finds analytically proven FISs for all of them, while the state-of-the-art baseline solution finds none. Our solution also achieves significantly smaller time costs on 40 of the 45 evaluated AI-controlled systems; and the centers of the found FISs also roughly meet the preferences set by domain experts.

The rest of the paper is organized as follows. Section~\ref{sect:RelatedWork} discusses related work. Section~\ref{sect:Background} introduces background technologies. Section~\ref{sect:ProblemFormulation} formulates the problem.  Section~\ref{sect:Solution} proposes and analyzes our solution. Section~\ref{sect:Evaluation} evaluates our solution.  Section~\ref{sect:Conclusion} concludes the paper.

%=============================
\section{Related Work}
\label{sect:RelatedWork}

How to find FIS (often in the form of Lyapunov stability study) is a long time topic in traditional control theory~\cite{khalil_nonlinear_2002}\cite{blanchini1999set}. 
For linear control systems, the most classic method is the \emph{Linear Matrix Inequalities} (LMI) method~\cite{duan2013}\cite{boyd1994}\cite{gahinet1994}. However, this method requires purely linear dynamics, which is not the case for nonlinear AI-controlled systems.
For nonlinear control systems, there is not a generic panacea solution for finding FIS~\cite{khalil2015}\cite{giesl2015}. The state-of-the-art methods include Lyapunov's indirect methods~\cite{khalil_nonlinear_2002}\cite{slotine1991applied}, Zubov's methods~\cite{zubov1964}\cite{camilli2001SIAM}, sum of squares programming~\cite{papachristodoulou2002cdc}\cite{Topcu2007ACC}, LaSalle's invariance principle~\cite{khalil_nonlinear_2002}\cite{lasalle1960IRE}, Krasovskii methods~\cite{krasovskiui1963}, and variable gradient methods~\cite{schultz1962}\cite{hang1970TAC}. However, these methods assume linear, polynomial, or other analytically tractable closed-form algebraic representations of the dynamics. Unfortunately, such assumptions do not suit AI-controlled systems involving NN components, as NN is well-known to be nonlinear, non-polynomial, and analytically intractable~\cite{prince2023}.

To avoid the analytically intractable problem, there are efforts to exploit formal method tools, such as the Satisfiability Modulo Theories (SMT)~\cite{zhou2022NIPS}, Mixed Integer Programming (MIP)~\cite{schwan2023TAC}, Hamilton-Jacobi (HJ) reachability analysis~\cite{ganai2024}, and Semi-Definite Programming (SDP)~\cite{yin2021TAC} to find FISs for NNs. However, SMT, MIP, and HJ methods have the well-known scalability problem; while SDP requires that the NN can be well approximated with smooth quadratic envelopes.

To address the above limitations, the state-of-the-art FIS finding method for AI-controlled system is the method proposed by Yang et al.~\cite{yang2024icml}. But as discussed in Section~\ref{sect:Introduction}, Yang et al.'s method is not for the cases where the controller is given a priori. It modifies the controller together with the FIS candidate until a FIS is found. In this paper, however, we focus on a different application context: a legal/regulatory authority needs to certify a vendor provided AI-controlled system. Therefore, the controller is given a priori. $\hfill$\myitem\label{myitem:why_yang_et_al_method_does_not_fit}

In case the AI-controlled system is given a priori, Harapanahalli et al.~\cite{harapanahalli2023csl} proposed a method to analytically prove if a given parallelotope FIS candidate is actually a FIS. But if the FIS candidate cannot be proven to be actually a FIS, the method does not specify how to search for another FIS candidate. That is, the method provides a sufficient condition to prove if a FIS candidate is actually a FIS; but it is not a method to find (proactively search for) FISs.

%=============================
\section{Background}
\label{sect:Background}

%........................................
\subsection{Invertible NN (INN) and Normalized INN (NINN)}
\label{subsect:InvertibleNnInnAndNormalizedInnNinn}

An INN~\cite{Laurent2017iclr} is an NN that learns a bijective mapping between the original $n$-dimensional state space $\mathbb{X}_{0}$ and the latent $n$-dimensional state space $\mathbb{X}_{1}$ (see (par\ref{myitem:heuristics_use_inn})). Once successfully trained, an INN can be denoted as a tuple of $(\mathcal{T}, \mathcal{T}^{-1})$, where $\mathcal{T}$ is a function of $\mathbb{X}_{0} \to \mathbb{X}_{1}$, while $\mathcal{T}^{-1}: \mathbb{X}_{1} \to \mathbb{X}_{0}$ is the inverse function of $\mathcal{T}$. Note, as ``bijective mappings,'' the range of $\mathcal{T}$ is the full set of $\mathbb{X}_{1}$, and vice versa for $\mathcal{T}^{-1}$. 

To speed up training, we can also insert a normalization layer $\mathcal{N}_{\mathsf{norm}}(\vec{\mu}, \vec{\sigma})$ between $\mathbb{X}_{0}$ and an INN $\mathcal{N}_{\mathsf{inn}} = (\mathcal{T}, \mathcal{T}^{-1})$, where $\vec{\mu} \in \mathbb{X}_{0}$, $\vec{\sigma} \in \mathbb{R}_{>0}^{n}$ are the configuration parameters given a priori. $\mathcal{N}_{\mathsf{norm}}(\vec{\mu}, \vec{\sigma})$ also consists of a transformation $\hat{\mathcal{T}}_{\vec{\mu}, \vec{\sigma}}$ and its inverse $\hat{\mathcal{T}}^{-1}_{\vec{\mu}, \vec{\sigma}}$, which together set up a bijective mapping between $\mathbb{X}_{0}$ and another (aka \emph{normalized}) $n$-dimensional state space $\hat{\mathbb{X}}_{0}$. 
Specifically,
\begin{eqnarray}
\forall \vec{x}_{0} \in \mathbb{X}_{0}, && \hat{\mathcal{T}}_{\vec{\mu}, \vec{\sigma}}(\vec{x}_{0}) 
\sdef 
(\vec{x}_{0} - \vec{\mu}) \oslash \vec{\sigma}, \nonumber \\
\forall \hat{\vec{x}}_{0} \in \hat{\mathbb{X}}_{0},  && \hat{\mathcal{T}}^{-1}_{\vec{\mu}, \vec{\sigma}}(\hat{\vec{x}}_{0})
\sdef
\hat{\vec{x}}_{0} \odot \vec{\sigma} + \vec{\mu}. \label{eqn:normalization}
\end{eqnarray}
\noindent where $\oslash$ and $\odot$ are respectively the component-wise division operator and the component-wise multiplication operator between two vectors. 

Correspondingly, $\mathcal{N}_{\mathsf{inn}} = (\mathcal{T}, \mathcal{T}^{-1})$ is now in charge of the bijective mapping between $\hat{\mathbb{X}}_{0}$ and $\mathbb{X}_{1}$.

We call the concatenated NN of $\mathcal{N}_{\mathsf{norm}}(\vec{\mu}, \vec{\sigma})$ and $\mathcal{N}_{\mathsf{inn}} = (\mathcal{T}, \mathcal{T}^{-1})$ the \emph{Normalized INN} (NINN), and denote it as $\tilde{\mathcal{N}}(\vec{\mu}, \vec{\sigma}, \mathcal{T}, \mathcal{T}^{-1})$. As $\mathcal{N}_{\mathsf{norm}}(\vec{\mu}, \vec{\sigma})$ is bijective between $\mathbb{X}_{0}$ and $\hat{\mathbb{X}}_{0}$, and $\mathcal{N}_{\mathsf{inn}} = (\mathcal{T}, \mathcal{T}^{-1})$ is bijective between $\hat{\mathbb{X}}_{0}$ and $\mathbb{X}_{1}$, $\tilde{\mathcal{N}}(\vec{\mu}, \vec{\sigma}, \mathcal{T}, \mathcal{T}^{-1})$ hence is bijective between $\mathbb{X}_{0}$ and $\mathbb{X}_{1}$. Denote the corresponding transformation from $\mathbb{X}_{0}$ to $\mathbb{X}_{1}$ and the inverse respectively as $\tilde{\mathcal{T}}$ and $\tilde{\mathcal{T}}^{-1}$, we have
\begin{eqnarray}
\forall \vec{x}_{0} \in \mathbb{X}_{0}, && 
\tilde{\mathcal{T}}(\vec{x}_{0}) = \mathcal{T}(\hat{\mathcal{T}}_{\vec{\mu}, \vec{\sigma}}(\vec{x}_{0})), \nonumber \\
\forall \vec{x}_{1} \in \mathbb{X}_{1}, && 
\tilde{\mathcal{T}}^{-1}(\vec{x}_{1}) = \hat{\mathcal{T}}^{-1}_{\vec{\mu}, \vec{\sigma}}(\mathcal{T}^{-1}(\vec{x}_{1})). 
\nonumber
\end{eqnarray}

Correspondingly, heuristics (par\ref{myitem:heuristics_use_inn})(par\ref{myitem:heuristics_fix_fis_adjust_inn}) and Fig.~\ref{fig:overall_framework} are adjusted: instead of training an INN, we shall train a NINN. Furthermore, when the context is unambiguous, we also denote NINN $\tilde{\mathcal{N}}(\vec{\mu}, \vec{\sigma}, \mathcal{T}, \mathcal{T}^{-1})$ with the tuple $(\tilde{\mathcal{T}}, \tilde{\mathcal{T}}^{-1})$. $\hfill$\myitem\label{myitem:heuristics_fix_fis_adjust_ninn}

%........................................
\subsection{Analytical Bound of NN Outputs: $\alpha,\beta$-CROWN}
\label{subsect:AnalyticalBoundOfNnOutputsAlphaBetaCrown}

$\alpha,\beta$-CROWN~\cite{zhang2018nips}\cite{wang2021betanips}\cite{xu2020fast}\cite{xu2020nips} is a state-of-the-art verification tool to analytically prove if a given polytope set $\mathcal{B}$ bounds the outputs of a complex function (typically an NN, or consisting of NN component(s)) over a given \emph{hyper-rectangular} (see \eqref{eqn:hyper-rectangle}) input domain .

$\alpha,\beta$-CROWN allows the complex function to conduct various computations, including various NN layering operations (fully connected, convolutional, pooling, etc.), various activation/nonlinear functions (ReLU, sigmoid, tanh, arctan, sin, cos, tan, gelu, pow, multiplication etc.), any linear specifications on the NN output etc.~\cite{zhang2018nips}\cite{wang2021betanips}\cite{xu2020fast}\cite{xu2020nips}.  

Specifically, $\alpha,\beta$-CROWN takes four inputs, respectively denoted as $f$, $\mathcal{X}$, $\mathcal{B}$, and $T_{\mathsf{max}}$ in the rest of this subsection. 

The first input to $\alpha,\beta$-CROWN is the complex function $f: \mathbb{R}^{n_{0}} \to \mathbb{R}^{n_{1}}$ to be analyzed, where $n_{0}$, $n_{1}$ are respectively the dimensions of the input and output vector spaces for $f$. Note $n_{0} = n_{1}$ is allowed.

The second input to $\alpha,\beta$-CROWN is the given input domain of concern $\mathcal{X} \subseteq \mathbb{R}^{n_{0}}$ for $f$. Note  $\mathcal{X}$ must be \emph{hyper-rectangular}, 
\begin{eqnarray}
\mbox{i.e., } \qquad \mathcal{X} = \{\vec{x} \mid \vec{x} \in \mathbb{R}^{n_{0}} \mbox{ and } \underline{\vec{x}} \preccurlyeq \vec{x} \preccurlyeq \overline{\vec{x}}\}, \label{eqn:hyper-rectangle}
\end{eqnarray}
\noindent where $\underline{\vec{x}} \in \mathbb{R}^{n_{0}}$ and $\overline{\vec{x}} \in \mathbb{R}^{n_{0}}$ respectively denote the lower and upper bounds of the input $\vec{x}$ to $f$; ``$\preccurlyeq$'' is the \emph{component-wise vector not-greater-than} operator, meaning $\underline{\vec{x}}[i] \leqslant \vec{x}[i] \leqslant \overline{\vec{x}}[i]$ ($\forall i \in \{1$, $2$, $\dots$, $n_{0}\}$), where $\vec{\bullet}[i]$ means the $i$th component of vector $\vec{\bullet}$.
%; the same applies to the \emph{component-wise vector not-less-than} operator ``$\succcurlyeq$''.

The third input to $\alpha,\beta$-CROWN is a given candidate output bound set $\mathcal{B} \subseteq \mathbb{R}^{n_{1}}$ in $f$'s output vector space. Note $\mathcal{B}$ must be \emph{polytope}, i.e.
\begin{eqnarray}
\mathcal{B} = \{\vec{y} \mid \vec{y} \in \mathbb{R}^{n_{1}} \mbox{ and } C_{\mathcal{B}} \vec{y} \preccurlyeq \vec{d}_{\mathcal{B}} \}, \label{eqn:polytope}
\end{eqnarray}
\noindent where $C_{\mathcal{B}} \in \mathbb{R}^{r \times n_{1}}$ and $\vec{d}_{\mathcal{B}} \in \mathbb{R}^{r}$. We can interpret $\mathcal{B}$ as $r$ polytope \emph{output constraints}, where the $j$th ($j = 1$, $2$, $\dots$, $r$) constraint is
\begin{eqnarray}
C_{\mathcal{B}}[j,:]\vec{y} \leqslant \vec{d}_{\mathcal{B}}[j],
\end{eqnarray}
\noindent where $\bullet[j,:]$ means the $j$th row of matrix $\bullet$. Note in case $n_{0} = n_{1}$, $\mathcal{B}$ can be $\mathcal{X}$, as hyper-rectangle is also polytope. 

The fourth input to $\alpha,\beta$-CROWN is the maximum computing time limit $T_{\mathsf{max}}$. Verification will be terminated with the result ``\emph{unknown}'' if this computing time limit is reached.

Given the above inputs, if $\alpha,\beta$-CROWN terminates within $T_{\mathsf{max}}$ with the result ``\emph{true},'' then it is \emph{analytically proven}~\cite{zhang2018nips}\cite{wang2021betanips}\cite{xu2020fast}\cite{xu2020nips} that 
\begin{eqnarray}
\forall \vec{x} \in \mathcal{X}, f(\vec{x}) \in \mathcal{B}; 
\mbox{ i.e., } 
\forall \vec{x}\in\mathcal{X}, C_{\mathcal{B}}f(\vec{x}) \preccurlyeq \vec{d}_{\mathcal{B}}. \label{eqn:true_result} 
\end{eqnarray}
If $\alpha,\beta$-CROWN terminates within $T_{\mathsf{max}}$ with the result ``\emph{false},'' then it is analytically proven~\cite{zhang2018nips}\cite{wang2021betanips}\cite{xu2020fast}\cite{xu2020nips} that  
\begin{eqnarray}
\exists \vec{x} \in \mathcal{X}, f(\vec{x}) \notin \mathcal{B}; 
\mbox{ i.e., } 
\exists \vec{x}\in\mathcal{X}, C_{\mathcal{B}}f(\vec{x}) \not\preccurlyeq \vec{d}_{\mathcal{B}}. 
\nonumber
\end{eqnarray}
If $\alpha,\beta$-CROWN terminates within $T_{\mathsf{max}}$ with the result ``\emph{unknown},'' then $\alpha,\beta$-CROWN realizes that it is unable to determine whether the result should be ``true'' or ``false.''  

In summary, we can denote $\alpha,\beta$-CROWN as a function $\alpha$\_$\beta$\_$\mathsf{CROWN}$($f$, $\mathcal{X}$, $\mathcal{B}$, $T_{\mathsf{max}}$), which returns ``true,'' ``false,'' or ``unknown'' within $T_{\mathsf{max}}$. $\hfill$\myitem\label{myitem:alpha_beta_CROWN_function_notation}

As $\alpha,\beta$-CROWN evolves, the chance that it can give deterministic answers (true or false) within $T_{\mathsf{max}}$ is improving~\cite{zhang2018nips}\cite{wang2021betanips}\cite{xu2020fast}\cite{xu2020nips}.

%=============================
\section{Problem Formulation}
\label{sect:ProblemFormulation}

Our problem formulation shall start from clarifying the details on an AI-controlled system $S$. 

First, a controlled system must have a \emph{controller}, and the controlled target. The latter is called the ``\emph{plant}.'' The aforementioned ``\emph{state}'' (see (par\ref{myitem:state})), rigorously speaking, should mean the \emph{state of the plant}. The controller takes as input the state of the plant $\vec{x}(t) \in \mathbb{R}^{n}$ (where $t$ is the time of concern), and outputs an \emph{actuation}, denoted as $\vec{u}(t) \in \mathbb{R}^{m}$, where $m$ is the dimension of the \emph{actuation space}. Thus, the controller is actually a function $\pi(\vec{x}(t)): \mathbb{R}^{n} \to \mathbb{R}^{m}$, where 
\begin{eqnarray}
\vec{u}(t) = \pi(\vec{x}(t)). 
\nonumber
\end{eqnarray}

When the controller outputs the actuation $\vec{u}(t)$, it is applied back to the plant, and the plant's state evolves accordingly. Formally, this means 
\begin{eqnarray}
\dot{\vec{x}} = f_{\mathsf{plt}}(\vec{x}(t), \vec{u}(t)) 
= f_{\mathsf{plt}}(\vec{x}(t), \pi(\vec{x}(t))) \sdef f_{\mathsf{cl}}(\vec{x}(t)), \label{eqn:dot_x}
\end{eqnarray}
\noindent where $\dot{\vec{x}}$ is the time derivative of $\vec{x}(t)$, and $f_{\mathsf{plt}}$ and $f_{\mathsf{cl}}$ are respectively called the \emph{plant model} and the \emph{closed-loop dynamics}.

When the controller is an AI-controller, and AI refers to an NN (see (par\ref{myitem:ai_is_nn_ai})), then the controller is typically run on discrete computers. 
Correspondingly, there are mature routines to convert the \emph{continuous time model} of \eqref{eqn:dot_x} to the below \emph{discretized model}~\cite{astrom1997}:
\begin{eqnarray}
\vec{x}_{\mathsf{d}}(k+1) &=& f_{\mathsf{pltd}}(\vec{x}_{\mathsf{d}}(k), \vec{u}_{\mathsf{d}}(k)) 
= f_{\mathsf{pltd}}(\vec{x}_{\mathsf{d}}(k), \pi(\vec{x}_{\mathsf{d}}(k))) \nonumber \\
&=& f_{\mathsf{cld}}(\vec{x}_{\mathsf{d}}(k)). \label{eqn:x_k+1}
\end{eqnarray}
\noindent where $k = 0$, $1$, $\dots$; $\vec{x}_{\mathsf{d}}(k) \sdef \vec{x}(t_{0}+k\delta)$ is the \emph{discretized state}; $t_{0}$ is the initial time instance; $\delta$ is the sensing/actuation period;
\begin{eqnarray}
\vec{u}_{\mathsf{d}}(k) \sdef \pi(\vec{x}_{\mathsf{d}}(k)) \label{eqn:definition_discretized_controller}
\end{eqnarray} 
\noindent is the \emph{discretized actuation}; note $\forall t \in [t_{0}+k\delta, t_{0}+(k+1)\delta)$ ($k = 0$, $1$, $\dots$), $\vec{u}(t) = \vec{u}_{\mathsf{d}}(k)$, i.e. the actuation is \emph{zero-order hold}; and $f_{\mathsf{pltd}}$ and $f_{\mathsf{cld}}$ are respectively called the \emph{discretized plant model} and the \emph{discretized closed-loop dynamics}.  \phantom{x} $\hfill$\myitem\label{myitem:zero_hold}

Typically, an AI-controller $\pi$ is directly trained upon a discretized plant model $f_{\mathsf{pltd}}$ using reinforcement learning~\cite{sutton1998reinforcement}, instead of trained on a physical plant (which follows the continuous time plant model $f_{\mathsf{plt}}$, see \eqref{eqn:dot_x}). In this way, the training can be conducted entirely on computer. This training preference is due to several reasons. First, AI-controller training typically carries out millions of operations on the plant. If conducted on the physical plant, so many operations can damage the plant, or cost too much time (computer simulated world's time can be orders of magnitude faster than the physical world time). Second, during the training, the immature AI-controller may issue dangerous operations, which may cause serious accidents if the training is conducted on the physical plant. $\hfill$\myitem\label{myitem:controller_training_uses_discretized_plant_model}

Back to the discretized model of \eqref{eqn:x_k+1}, a \emph{discretized FIS} $\mathcal{I}_{\mathsf{d}}$ is defined as a subset of the $n$-dimensional state space $\mathbb{R}^{n}$, where given any $\vec{x}_{\mathsf{d}}(0) \in \mathcal{I}_{\mathsf{d}}$, the future \emph{discretized state trajectory} $\{\vec{x}_{\mathsf{d}}(k)\}_{k = 0, 1, \dots}$ derived from \eqref{eqn:x_k+1} always remains in $\mathcal{I}_{\mathsf{d}}$. $\hfill$\myitem\label{myitem:definition_discretized_fis} 

If we can find a hyper-rectangular discretized FIS $\mathcal{I}_{\mathsf{d}}$, and given  $\vec{x}_{\mathsf{d}}(k) \in \mathcal{I}_{\mathsf{d}}$ ($\forall k \in \mathbb{Z}_{\geqslant 0}$, note $\mathcal{I}_{\mathsf{d}}$ is independent from $\vec{x}_{\mathsf{d}}(k)$), then we have mature tools~\cite{zhang2018nips}\cite{wang2021betanips}\cite{xu2020fast}\cite{xu2020nips} to analyze \eqref{eqn:definition_discretized_controller} to find a hyper-rectangular bound $\mathcal{B}_{\mathsf{ad}}$ (independent from $\vec{x}_{\mathsf{d}}(k)$) for $\vec{u}_{\mathsf{d}}(k)$. With $\mathcal{B}_{\mathsf{ad}}$ and zero-order hold actuation (see (par\ref{myitem:zero_hold})), there are often (depending on \eqref{eqn:dot_x}) mature routines~\cite{astrom1997} to find the hyper-rectangular bound $\mathcal{B}_{\mathsf{ed}}$ (independent from $\vec{x}_{\mathsf{d}}(k)$) on the \emph{discretization error}: $(\vec{x}(t) - \vec{x}_{\mathsf{d}}(k))$ ($\forall t \in [t_{0}+k\delta, t_{0}+(k+1)\delta)$). In this paper, we shall focus on the cases where such mature routines to derive $\mathcal{B}_{\mathsf{ed}}$ from  $\mathcal{B}_{\mathsf{ad}}$ exist. Such cases are common when the plant is a macroscopic physical world object~\cite{astrom1997}.  \phantom{xxx} $\hfill$\myitem\label{myitem:discretization_error_bound}

For the aforementioned cases, once a hyper-rectangular discretized FIS $\mathcal{I}_{\mathsf{d}}$ independent from $\vec{x}_{\mathsf{d}}(k)$ ($\forall k \in \mathbb{Z}_{\geqslant 0}$) is found for the discretized model of \eqref{eqn:x_k+1}, then $\forall \vec{x}(t_{0}) = \vec{x}_{\mathsf{d}}(0) \in \mathcal{I}_{\mathsf{d}}$, the future discretized trajectory $\{\vec{x}_{\mathsf{d}}(k)\}_{k = 0, 1, \dots}$ is bounded by $\mathcal{I}_{\mathsf{d}}$, and the future continuous trajectory $\{\vec{x}(t)\}_{t \in [t_{0}, +\infty)}$ is bounded by $(\mathcal{I}_{\mathsf{d}} + \mathcal{B}_{\mathsf{ed}})$ (here ``$+$'' means the vector sum of sets, aka Minkowski sum). Therefore, \emph{as long as $(\mathcal{I}_{\mathsf{d}}+\mathcal{B}_{\mathsf{ed}}) \cap \mathcal{F} = \varnothing$, the continuous system is safe.} \phantom{xxx} $\hfill$\myitem\label{myitem:continuous_system_safety}

(par\ref{myitem:continuous_system_safety}) and (par\ref{myitem:discretization_error_bound}) show that the key/only missing piece in guaranteeing proven safety of an AI-controlled system $S$ lies in finding $\mathcal{I}_{\mathsf{d}}$, i.e. a hyper-rectangular discretized FIS for the discretized model \eqref{eqn:x_k+1} of $S$. Therefore, 
in the following, \emph{we shall focus on finding this $\mathcal{I}_{\mathsf{d}}$}. \phantom{xxx} $\hfill$\myitem\label{myitem:focus_problem}

Once this $\mathcal{I}_{\mathsf{d}}$ is found, the \emph{proven safety} guarantee shall mean $(\mathcal{I}_{\mathsf{d}}+\mathcal{B}_{\mathsf{ed}}) \cap \mathcal{F} = \varnothing$ (see (par\ref{myitem:continuous_system_safety})) instead of $\mathcal{I} \cap \mathcal{F} = \varnothing$ (see (par\ref{myitem:how_to_let_authority_trust})). An exception is when $S$ is discrete (i.e. of model \eqref{eqn:x_k+1}) to begin with, then the proven safety guarantee means $\mathcal{I}_{\mathsf{d}} \cap \mathcal{F} = \varnothing$. Also note that when the plant is a macroscopic physical world object, often $\mathsf{lim}_{\delta \rightarrow 0} \mathcal{B}_{\mathsf{ed}} = \varnothing$~\cite{astrom1997}, hence $\mathsf{lim}_{\delta \rightarrow 0}(\mathcal{I}_{\mathsf{d}} + \mathcal{B}_{\mathsf{ed}}) = \mathcal{I}_{\mathsf{d}}$. Therefore, with a sufficiently small sensing/actuation period $\delta$, the intended proven safety guarantee of $(\mathcal{I}_{\mathsf{d}} + \mathcal{B}_{\mathsf{ed}}) \cap \mathcal{F} = \varnothing$ can be empirically relaxed to $\mathcal{I}_{\mathsf{d}} \cap \mathcal{F} = \varnothing$. \phantom{xxx} $\hfill$\myitem\label{myitem:discretization_proven_safety}

Due to (par\ref{myitem:controller_training_uses_discretized_plant_model})(par\ref{myitem:focus_problem})(par\ref{myitem:discretization_proven_safety}), we shall focus on the discretized state, discretized actuation, discretized plant model, discretized closed-loop dynamics, discretized model, discretized FIS, and discretized state trajectory, instead of their continuous time counterparts. Hence, to simplify narration, unless otherwise denoted, \emph{we no longer explicitly add the prefix ``discretized'' when referring to these concepts}.

%=============================
\section{Solution}
\label{sect:Solution}

As per Section~\ref{sect:Introduction} and \ref{sect:ProblemFormulation}, this paper shall focus on finding a hyper-rectangular  FIS $\mathcal{I}_{\mathsf{d}}$ for the  model \eqref{eqn:x_k+1} of $S$ (see (par\ref{myitem:focus_problem})), where the AI-controller $\pi$ is given a priori (see (par\ref{myitem:focus_of_this_paper})). 

%-------------------------------------
%\subsection{Verification Formulation}
%
%\label{subsect:VerificationFormulation}

However, as per heuristics (par\ref{myitem:heuristics_use_inn})(par\ref{myitem:heuristics_fix_fis_adjust_inn})(par\ref{myitem:heuristics_fix_fis_adjust_ninn}), the chance that the highly nonlinear system $S$ has a hyper-rectangular FIS is low. To remedy, we propose to construct a transformation function $\tilde{\mathcal{T}}$ and its inverse $\tilde{\mathcal{T}}^{-1}$, to transform the \emph{original} $n$-dimensional state space $\mathbb{X}_{0}$ to another (aka \emph{latent}) $n$-dimensional state space $\mathbb{X}_{1}$, so that $S$'s nonlinear behaviors in $\mathbb{X}_{0}$ become more \emph{regular} in $\mathbb{X}_{1}$: specifically, a hyper-rectangular  FIS $\mathcal{I}_{\mathsf{d},1}$ can be found in $\mathbb{X}_{1}$ for $S$. \phantom{x} $\hfill$\myitem\label{myitem:definition_regular}  

Formally, $\tilde{\mathcal{T}}: \mathbb{X}_{0} \to \mathbb{X}_{1}$ and $\tilde{\mathcal{T}}^{-1}: \mathbb{X}_{1} \to \mathbb{X}_{0}$ is a pair of transformation functions between $n$-dimensional state spaces $\mathbb{X}_{0}$ and $\mathbb{X}_{1}$, such that
\begin{eqnarray}
\forall \vec{x}_{0} \in \mathbb{X}_{0}, \vec{x}_{1} = \tilde{\mathcal{T}}(\vec{x}_{0}) \in \mathbb{X}_{1} \mbox{ and } \vec{x}_{0} = \tilde{\mathcal{T}}^{-1}(\vec{x}_{1}); \nonumber \\
\forall \vec{x}_{1} \in \mathbb{X}_{1}, \vec{x}_{0} = \tilde{\mathcal{T}}^{-1}(\vec{x}_{1}) \in \mathbb{X}_{0} \mbox{ and } \vec{x}_{1} = \tilde{\mathcal{T}}(\vec{x}_{0}). \label{eqn:definition_inverse_transformations} 
\end{eqnarray}

Given $\mathcal{X}_{0} \subseteq \mathbb{X}_{0}$ or $\mathcal{X}_{1} \subseteq \mathbb{X}_{1}$, we further define
\begin{eqnarray}
\tilde{\mathcal{T}}(\mathcal{X}_{0}) &\sdef& \{ \vec{x}_{1} \mid \exists \vec{x}_{0} \in \mathcal{X}_{0} \mbox{ s.t. } \vec{x}_{1} = \tilde{\mathcal{T}}(\vec{x}_{0}) \}, \nonumber \\
\tilde{\mathcal{T}}^{-1}(\mathcal{X}_{1}) &\sdef& \{ \vec{x}_{0} \mid \exists \vec{x}_{1} \in \mathcal{X}_{1} \mbox{ s.t. } \vec{x}_{0} = \tilde{\mathcal{T}}^{-1}(\vec{x}_{1}) \}. \quad \label{eqn:definition_transformed_set} 
\end{eqnarray}

We have a trivial proposition for \eqref{eqn:definition_transformed_set}: 

\begin{proposition}
\label{proposition:latent_state_of_a_X0_state_is_in_X1}
\textbf{Claim 1}) $\forall \vec{x}_{0} \in \tilde{\mathcal{T}}^{-1}(\mathcal{X}_{1})$, $\tilde{\mathcal{T}}(\vec{x}_{0}) \in \mathcal{X}_{1}$; \textbf{Claim 2}) $\forall \vec{x}_{1} \in \tilde{\mathcal{T}}(\mathcal{X}_{0})$, $\tilde{\mathcal{T}}^{-1}(\vec{x}_{1}) \in \mathcal{X}_{0}$. $\hfill\square$
\end{proposition}

\noindent\textit{Proof}: 
As $\vec{x}_{0} \in \tilde{\mathcal{T}}^{-1}(\mathcal{X}_{1})$, due to the definition of $\tilde{\mathcal{T}}^{-1}(\mathcal{X}_{1})$ (see \eqref{eqn:definition_transformed_set}), we have
\begin{eqnarray}
\exists \vec{x}_{1} \in \mathcal{X}_{1} \mbox{ s.t. } \vec{x}_{0} = \tilde{\mathcal{T}}^{-1}(\vec{x}_{1}). \label{eqn:exists_x_1}
\end{eqnarray}
Apply $\tilde{\mathcal{T}}$ to both sides of \eqref{eqn:exists_x_1}, we get
\begin{eqnarray}
\vec{x}_{1} = \tilde{\mathcal{T}}(\vec{x}_{0}). \label{eqn:x_1}
\end{eqnarray}

Hence $\vec{x}_{1} \in \mathcal{X}_{1} \mbox{ (see \eqref{eqn:exists_x_1})} \Rightarrow \tilde{\mathcal{T}}(\vec{x}_{0}) \in \mathcal{X}_{1}$. Claim 1 holds. 

Claim 2 can be proved the same way. $\hfill\blacksquare$

\vspace{0.5ex}

Denote the  states $\vec{x}_{\mathsf{d}}$ in the original state space $\mathbb{X}_{0}$ as $\vec{x}_{\mathsf{d},0}$, $f_{\mathsf{pltd}}$ as $f_{\mathsf{pltd}, 0}$, $f_{\mathsf{cld}}$ as $f_{\mathsf{cld},0}$, $\vec{u}_{\mathsf{d}}$ as $\vec{u}_{\mathsf{d},0}$, and $\pi$ as $\pi_{0}$, we can rewrite $S$'s  model \eqref{eqn:x_k+1} in $\mathbb{X}_{0}$ as
\begin{eqnarray}
&& \vec{x}_{\mathsf{d},0}(k+1) = f_{\mathsf{pltd},0}(\vec{x}_{\mathsf{d},0}(k), \vec{u}_{\mathsf{d},0}(k)) \nonumber \\
&=& f_{\mathsf{pltd},0}(\vec{x}_{\mathsf{d},0}(k), \pi_{0}(\vec{x}_{\mathsf{d},0}(k))) 
= f_{\mathsf{cld},0}(\vec{x}_{\mathsf{d},0}(k)). \quad \label{eqn:x_k+1_original_state_space}
\end{eqnarray}

For $k = 0$, $1$, $\dots$, define
\begin{eqnarray}
\vec{x}_{\mathsf{d},1}(k) \sdef \tilde{\mathcal{T}}(\vec{x}_{\mathsf{d},0}(k)) \, \mbox{ and }  \, \vec{u}_{\mathsf{d},1}(k) \sdef \tilde{\mathcal{T}}(\vec{u}_{\mathsf{d},0}(k)). 
\nonumber
\end{eqnarray}

Apply $\tilde{\mathcal{T}}$ to both sides of \eqref{eqn:x_k+1_original_state_space}, and combine with \eqref{eqn:definition_inverse_transformations}, we get $S$'s  model in the latent state space $\mathbb{X}_{1}$ as 
\begin{eqnarray}
&& \vec{x}_{\mathsf{d},1}(k+1) = \tilde{\mathcal{T}}(\vec{x}_{\mathsf{d},0}(k+1)) \nonumber \\
&=& \tilde{\mathcal{T}}(f_{\mathsf{cld},0}(\tilde{\mathcal{T}}^{-1}(\vec{x}_{\mathsf{d},1}(k)))) \sdef f_{\mathsf{cld},1}(\vec{x}_{\mathsf{d},1}(k)). \quad  \label{eqn:x_k+1_latent_state_space}
\end{eqnarray}
\noindent Similarly, we can derive the discretized actuation in $\mathbb{X}_{1}$ to be  
\begin{eqnarray}
\vec{u}_{\mathsf{d},1}(k) 
= \tilde{\mathcal{T}}(\pi_{0}(\tilde{\mathcal{T}}^{-1}(\vec{x}_{\mathsf{d},1}(k))))
\sdef \pi_{1}(\vec{x}_{\mathsf{d},1}(k)). \nonumber 
\end{eqnarray}

\vspace{1ex} 

We have the following proposition, lemma, and theorem. 

\begin{proposition}
\label{proposition:SufficientInductiveConditionForFis}
For an AI-controlled system $S$ of  model \eqref{eqn:x_k+1}, a set $\mathcal{I}_{\mathsf{d}} \subseteq \mathbb{R}^{n}$ is a  FIS of $S$ if $\forall \vec{x}_{\mathsf{d}}(0) \in \mathcal{I}_{\mathsf{d}}$, $\vec{x}_{\mathsf{d}}(1) \in \mathcal{I}_{\mathsf{d}}$. $\hfill\square$
\end{proposition}

\noindent\textit{Proof}:
Regard $\vec{x}_{\mathsf{d}}(1)$ as the new $\vec{x}_{\mathsf{d}}(0)$, and the induction can go on forever. $\hfill\blacksquare$

\begin{lemma}
\label{lemma:SufficientConditionForFISinOriginalStateSpace}
If $\mathcal{I}_{\mathsf{d},1} \subseteq \mathbb{X}_{1}$ is a  FIS for the  model \eqref{eqn:x_k+1_latent_state_space} of $S$ in state space $\mathbb{X}_{1}$, then $\mathcal{I}_{\mathsf{d},0} = \tilde{\mathcal{T}}^{-1}(\mathcal{I}_{\mathsf{d},1})$ is a  FIS for the  model \eqref{eqn:x_k+1_original_state_space} of $S$ in state space $\mathbb{X}_{0}$. $\hfill\square$
\end{lemma}

\noindent\textit{Proof}:
$\forall \vec{x}_{\mathsf{d},0}(0) \in \mathcal{I}_{\mathsf{d},0} = \tilde{\mathcal{T}}^{-1}(\mathcal{I}_{\mathsf{d},1})$, as per $S$'s  model \eqref{eqn:x_k+1_original_state_space} in the original state space $\mathbb{X}_{0}$, the next state in $\mathbb{X}_{0}$ should be
\begin{eqnarray}
\vec{x}_{\mathsf{d},0}(1) = f_{\mathsf{cld},0}(\vec{x}_{\mathsf{d},0}(0)). \label{eqn:x_d_0_1}
\end{eqnarray}
\noindent We need to prove $\vec{x}_{\mathsf{d}, 0}(1) \in \mathcal{I}_{\mathsf{d}, 0}$. To prove this, first denote 
\begin{eqnarray}
\vec{x}_{\mathsf{d},1}(0) \sdef \tilde{\mathcal{T}}(\vec{x}_{\mathsf{d},0}(0)). \label{eqn:x_d_1_0}
\end{eqnarray}
\noindent Due to Proposition~\ref{proposition:latent_state_of_a_X0_state_is_in_X1}, $\vec{x}_{\mathsf{d},1}(0) \in \mathcal{I}_{\mathsf{d}, 1}$.

As $\mathcal{I}_{\mathsf{d},1}$ is a  FIS for the  model \eqref{eqn:x_k+1_latent_state_space} of $S$ in state space $\mathbb{X}_{1}$, as per the definition of  FIS (see (par\ref{myitem:definition_discretized_fis})), we have 
\begin{equation}
\vec{x}_{\mathsf{d},1}(1) = f_{\mathsf{cld},1}(\vec{x}_{\mathsf{d},1}(0))  \in \mathcal{I}_{\mathsf{d},1}. \label{eqn:x_d_1_1_in_I_d_1}
\end{equation}

Meanwhile, due to \eqref{eqn:x_k+1_latent_state_space}, 
\begin{eqnarray}
\vec{x}_{\mathsf{d},1}(1) = f_{\mathsf{cld},1}(\vec{x}_{\mathsf{d},1}(0)) &=& \tilde{\mathcal{T}}(f_{\mathsf{cld},0}(\tilde{\mathcal{T}}^{-1}(\vec{x}_{\mathsf{d},1}(0)))), \nonumber \\
\mbox{due to \eqref{eqn:x_d_1_0}}, \quad \vec{x}_{\mathsf{d},1}(1) &=& \tilde{\mathcal{T}}(f_{\mathsf{cld},0}(\vec{x}_{\mathsf{d},0}(0))), \nonumber \\
\mbox{due to \eqref{eqn:x_d_0_1}}, \quad \vec{x}_{\mathsf{d},1}(1) &=& \tilde{\mathcal{T}}(\vec{x}_{\mathsf{d},0}(1)), \nonumber \\
\mbox{therefore} \quad \vec{x}_{\mathsf{d},0}(1) &=& \tilde{\mathcal{T}}^{-1}(\vec{x}_{\mathsf{d},1}(1)). \label{eqn:x_d_0_1_T_-1_x_d_1_1} \\
\eqref{eqn:x_d_1_1_in_I_d_1}, \eqref{eqn:x_d_0_1_T_-1_x_d_1_1}, \eqref{eqn:definition_transformed_set} \Rightarrow \vec{x}_{\mathsf{d},0}(1) &\in& \tilde{\mathcal{T}}^{-1}(\mathcal{I}_{\mathsf{d},1}) = \mathcal{I}_{\mathsf{d},0}. \label{eqn:x_d_0_1_in_I_d_0}
\end{eqnarray}

Therefore, in the state space $\mathbb{X}_{0}$, \eqref{eqn:x_d_0_1_in_I_d_0} proves that  $\forall \vec{x}_{\mathsf{d},0}(0) \in \mathcal{I}_{\mathsf{d},0}$, we have $\vec{x}_{\mathsf{d},0}(1) \in \mathcal{I}_{\mathsf{d},0}$. Due to Proposition~\ref{proposition:SufficientInductiveConditionForFis}, $\mathcal{I}_{\mathsf{d},0}$ hence is a  FIS for the  model \eqref{eqn:x_k+1_original_state_space} of $S$ in $\mathbb{X}_{0}$. \phantom{x} $\hfill\blacksquare$

\begin{theorem}
\label{theorem:HowToProveFIS}
Given an AI-controlled system $S$ of  model~\eqref{eqn:x_k+1_latent_state_space} in $\mathbb{X}_{1}$, whose  closed-loop dynamics is $f_{\mathsf{cld}, 1}$. 
Given a hyper-rectangular set $\mathcal{I}_{\mathsf{d}, 1} \subseteq \mathbb{X}_{1}$, such that $\alpha$\_$\beta$\_$\mathsf{CROWN}$($f$ = $f_{\mathsf{cld},1}$, $\mathcal{X}$ = $\mathcal{I}_{\mathsf{d}, 1}$, $\mathcal{B}$ = $\mathcal{I}_{\mathsf{d}, 1}$, $T_{\mathsf{max}}$) (see (par\ref{myitem:alpha_beta_CROWN_function_notation})) returns true. Then
\textbf{Claim 1}) $\mathcal{I}_{\mathsf{d},1}$ is an analytically proven  hyper-rectangular FIS for the  model \eqref{eqn:x_k+1_latent_state_space} of $S$ in $\mathbb{X}_{1}$; \textbf{Claim 2}) $\mathcal{I}_{\mathsf{d},0} = \tilde{\mathcal{T}}^{-1}(\mathcal{I}_{\mathsf{d},1})$ is an analytically proven  FIS for the  model \eqref{eqn:x_k+1_original_state_space} of $S$ in $\mathbb{X}_{0}$. $\hfill\square$
\end{theorem}

\noindent\textit{Proof}: 
As per \eqref{eqn:true_result} and Proposition~\ref{proposition:SufficientInductiveConditionForFis}, Claim~1 sustains. 

Based on Claim~1 and Lemma~\ref{lemma:SufficientConditionForFISinOriginalStateSpace}, $\mathcal{I}_{\mathsf{d},0} = \tilde{\mathcal{T}}^{-1}(\mathcal{I}_{\mathsf{d},1})$ is hence a  FIS for the  model \eqref{eqn:x_k+1_original_state_space} of $S$ in $\mathbb{X}_{0}$. \phantom{xxx} $\hfill\blacksquare$

\vspace{1ex}

With Theorem~\ref{theorem:HowToProveFIS}, we propose the following \emph{Overall Procedure} to realize the heuristics of (par\ref{myitem:definition_regular}), given the  model \eqref{eqn:x_k+1_original_state_space} of $S$ in the original $n$-dimensional state space $\mathbb{X}_{0}$. \phantom{x} $\hfill$\myitem\label{myitem:overall_procedure}

\vspace{1ex}

\setcounter{ctr}{0}

\noindent \textbf{Step \stp{step:choose_candidate_fis_in_latent_state_space}} In another (aka latent) $n$-dimensional state space $\mathbb{X}_{1}$, we forcefully choose a hyper-rectangular FIS candidate $\mathcal{I}_{\mathsf{d},1}$.

\noindent \textbf{Step \stp{step:train_inn}} Train a NINN $\tilde{\mathcal{N}} = (\tilde{\mathcal{T}}, \tilde{\mathcal{T}}^{-1})$, to transform between $\mathbb{X}_{0}$ and $\mathbb{X}_{1}$ (see \eqref{eqn:definition_inverse_transformations} for what this means). Note $\tilde{\mathcal{T}}$ and $\tilde{\mathcal{T}}^{-1}$ define the details of \eqref{eqn:x_k+1_latent_state_space}, i.e. the  model of $S$ in $\mathbb{X}_{1}$. In each \emph{training epoch} (simplified as ``\emph{epoch}'' in the following), the loss function and the sampled training set from $\mathcal{I}_{\mathsf{d},1}$ shall be designed to adjust $\tilde{\mathcal{T}}$ and $\tilde{\mathcal{T}}^{-1}$ (hence \eqref{eqn:x_k+1_latent_state_space}), so that $\mathcal{I}_{\mathsf{d},1}$ looks more like a  FIS for \eqref{eqn:x_k+1_latent_state_space}. 
Also note as per Section~\ref{subsect:InvertibleNnInnAndNormalizedInnNinn}, to train a NINN, we need to set the normalization configuration parameters $\vec{\mu}$ and $\vec{\sigma}$ a priori. The settings of $\vec{\mu}$ and $\vec{\sigma}$ are application dependent. Intuitively, they should respectively reflect the center and the dimension-wise radii of a preferred  FIS in $\mathbb{X}_{0}$. For example, if the safe set $\bar{\mathcal{F}}$ (see (par\ref{myitem:definition_set_of_concerned_unsafe_states})) in $\mathbb{X}_{0}$ is a hyper-rectangle, then we can set $\vec{\mu}$ to the center of this hyper-rectangle, and set every component of $\vec{\sigma}$ to the corresponding half-edge-length of this hyper-rectangle.

\noindent \textbf{Step \stp{step:check_if_candidate_fis_is_fis}} With the $\tilde{\mathcal{T}}$ and $\tilde{\mathcal{T}}^{-1}$ trained from Step~\ref{step:train_inn} defining  \eqref{eqn:x_k+1_latent_state_space}, use $\alpha,\beta$-CROWN to analytically verify if $\mathcal{I}_{\mathbf{d},1}$ is actually a hyper-rectangular FIS for \eqref{eqn:x_k+1_latent_state_space} in $\mathbb{X}_{1}$. Specifically, verify the result of $\alpha$\_$\beta$\_$\mathsf{CROWN}$($f$ = $f_{\mathsf{cld},1}$, $\mathcal{X}$ = $\mathcal{I}_{\mathsf{d},1}$, $\mathcal{B}$ = $\mathcal{I}_{\mathsf{d},1}$, $T_{\mathsf{max}}$) (see (par\ref{myitem:alpha_beta_CROWN_function_notation})). If the  result is true, then as per Theorem~\ref{theorem:HowToProveFIS} Claim 1, $\mathcal{I}_{\mathsf{d},1}$ is an \emph{analytically proven} hyper-rectangular FIS for the  model \eqref{eqn:x_k+1_latent_state_space} of $S$ in $\mathbb{X}_{1}$, and $\mathcal{I}_{\mathsf{d},0} = \tilde{\mathcal{T}}^{-1}(\mathcal{I}_{\mathsf{d},1})$ is an \emph{analytically proven}  FIS for the  model \eqref{eqn:x_k+1_original_state_space} of $S$ in $\mathbb{X}_{0}$. Return ``success.''  Otherwise (i.e. if the verification result is false or unknown), return ``failure.'' 

\vspace{1ex}

Several details of the Overall Procedure are further clarified by the following subsections.

%-----------------------------------------
\subsection{About Overall Procedure Step~\ref{step:choose_candidate_fis_in_latent_state_space}: How to Choose $\mathcal{I}_{\mathsf{d},1}$}
\label{subsect:OverallProcedureHowToChooseCandiddateFis}

In the latent state space $\mathbb{X}_{1}$, define the  \emph{centered hypercube with radius $R$} (where $R \in \mathbb{R}_{>0})$, denoted as $\mathcal{C}(R)$, as
\begin{eqnarray}
\mathcal{C}(R) \sdef \{\vec{x}_{1} \mid \vec{x}_{1} \in \mathbb{X}_{1}, \mbox{ and } \|\vec{x}_{1}\|_{\infty}\leqslant R\}, \label{eqn:centered_hypercube}
\end{eqnarray}
\noindent where $\| \vec{\bullet} \|_{\infty}$ is the $\infty$-norm (i.e. maximum absolute value of vector $\vec{\bullet}$'s components), and $R$ is the so-called \emph{hypercube radius}.
Note $\mathcal{C}(R)$ is centered upon the origin of $\mathbb{X}_{1}$; and as a hypercube, it is also a hyper-rectangle. 

We set $\mathcal{I}_{\mathsf{d},1}$, the hyper-rectangular FIS candidate in $\mathbb{X}_{1}$, as  
\begin{eqnarray}
\mathcal{I}_{\mathsf{d},1}
\sdef \mathcal{C}(R_{\mathsf{fis}}), 
\label{eqn:intended_latent_fis_box}
\end{eqnarray}
\noindent where $R_{\mathsf{fis}} \in \mathbb{R}_{>0}$ is the pre-configured hypercube radius constant.

\subsection{About Overall Procedure Step~\ref{step:train_inn}: How to Train the NINN}
\label{subsect:TrainingNinn}

As per heuristics (par\ref{myitem:heuristics_fix_fis_adjust_inn})(par\ref{myitem:heuristics_fix_fis_adjust_ninn}), we shall fix the hyper-rectangular FIS candidate $\mathcal{I}_{\mathsf{d},1}$ in $\mathbb{X}_{1}$, whereas properly train the NINN $\tilde{\mathcal{N}} = (\tilde{\mathcal{T}},\tilde{\mathcal{T}}^{-1})$ (hence adjust \eqref{eqn:x_k+1_latent_state_space}), so that $\mathcal{I}_{\mathsf{d},1}$ looks more like an actual  FIS for \eqref{eqn:x_k+1_latent_state_space}. 

To properly train $\tilde{\mathcal{N}}$, we need to punish deviations from  $\mathcal{I}_{\mathsf{d},1}$. Hence we need to first formally define ``deviation.'' Specifically, given a pre-configured \emph{trajectory horizon} constant $K \in \mathbb{Z}_{>0}$, $\forall \vec{x}_{\mathsf{d}, 1}(0) \in \mathbb{X}_{1}$, define $\vec{x}_{\mathsf{d}, 1}(0)$'s  \emph{$K$-step maximum deviation from $\mathcal{C}(R)$} (simplified as the  ``\emph{deviation}'' when the context is unambiguous) as 
\begin{eqnarray}
&& d^{*}(\vec{x}_{\mathsf{d}, 1}(0), K, \mathcal{C}(R)) \sdef  d^{*}(\vec{x}_{\mathsf{d}, 1}(0), K, R) \nonumber \\
&\sdef&
\smax\limits_{k = 1, \dots, K} \{(\|\vec{x}_{\mathsf{d}, 1}(k)\|_{\infty} - R)_{+}
\},
\label{eqn:K_step_maximum_deviation_from_mathcal_C_R}
\end{eqnarray}
where $\mathcal{C}(R)$ is a centered hypercube defined by \eqref{eqn:centered_hypercube}, 
$\vec{x}_{\mathsf{d}, 1}(k)$ (where $k = 1$, $\dots$, $K$) are defined by \eqref{eqn:x_k+1_latent_state_space}, and $(\bullet)_{+} \sdef \smax\{0, \bullet\}$. 
\begin{comment}
Correspondingly, we denote the $k$ value that maximizes $(\|\vec{x}_{\mathsf{d}, 1}(k)\|_{\infty} - R)_{+}$ in the right hand side of \eqref{eqn:K_step_maximum_deviation_from_mathcal_C_R} as $k^{*}(\vec{x}_{\mathsf{d}, 1}(0), K, R)$.
\end{comment}

Given $\mathcal{I}_{\mathsf{d}, 1} = \mathcal{C}(R_{\mathsf{fis}})$ (see \eqref{eqn:intended_latent_fis_box}), and given 
\begin{eqnarray}
R_{\mathsf{shr}} &\sdef& \rho_{\mathsf{shr}} R_{\mathsf{fis}},  \label{eqn:R_shr}
\end{eqnarray}
\noindent where $\rho_{\mathsf{shr}} \in (0, 1]$ is the pre-configured \emph{shrinking ratio} constant,  
if
\begin{eqnarray}
\forall \vec{x}_{\mathsf{d}, 1}(0) \in \mathcal{I}_{\mathsf{d}, 1} = \mathcal{C}(R_{\mathsf{fis}}), \; d^{*}(\vec{x}_{\mathsf{d}, 1}(0), K, R_{\mathsf{shr}}) = 0, \label{eqn:heuristics_candidate_fis_is_actually_a_fis}
\end{eqnarray}
\noindent then $\mathcal{I}_{\mathsf{d}, 1} = \mathcal{C}(R_{\mathsf{fis}})$ looks like a  FIS. In contrast, if
\begin{eqnarray}
\exists \vec{x}_{\mathsf{d}, 1}(0) \in \mathcal{I}_{\mathsf{d}, 1} = \mathcal{C}(R_{\mathsf{fis}}), \; d^{*}(\vec{x}_{\mathsf{d}, 1}(0), K, R_{\mathsf{shr}}) > 0, \label{eqn:heuristics_candidate_fis_is_not_actually_a_fis}
\end{eqnarray}
\noindent then $\mathcal{I}_{\mathsf{d}, 1} = \mathcal{C}(R_{\mathsf{fis}})$ looks unlike a  FIS. 

Our loss function shall be designed to encourage \eqref{eqn:heuristics_candidate_fis_is_actually_a_fis} and discourage \eqref{eqn:heuristics_candidate_fis_is_not_actually_a_fis}. $\hfill$\myitem\label{myitem:heuristics_deviation_loss_function}

To gather more info for the loss function, we prefer $K = +\infty$, but thus the computing time cost will surge to infinity. We have to compromise based on the available computing resources. Empirically, we can set $K$ to $3 \sim 5$. 

Similarly, to check every vector $\vec{x}_{\mathsf{d}, 1}(0) \in \mathcal{I}_{\mathsf{d}, 1} = \mathcal{C}(R_{\mathsf{fis}})$ on \eqref{eqn:heuristics_candidate_fis_is_actually_a_fis} and \eqref{eqn:heuristics_candidate_fis_is_not_actually_a_fis} is neither feasible due to the infinite computing time cost incurred. Therefore, in each epoch, we have to wisely sample a few subsets, aka \emph{training subsets}, of vectors from $\mathcal{I}_{\mathsf{d}, 1} = \mathcal{C}(R_{\mathsf{fis}})$, to form our training set.  Fig.~\ref{fig:training_samples} illustrates these training subsets. 

\begin{figure}[t]
\centering
\includegraphics[width=0.45\linewidth]{\detokenize{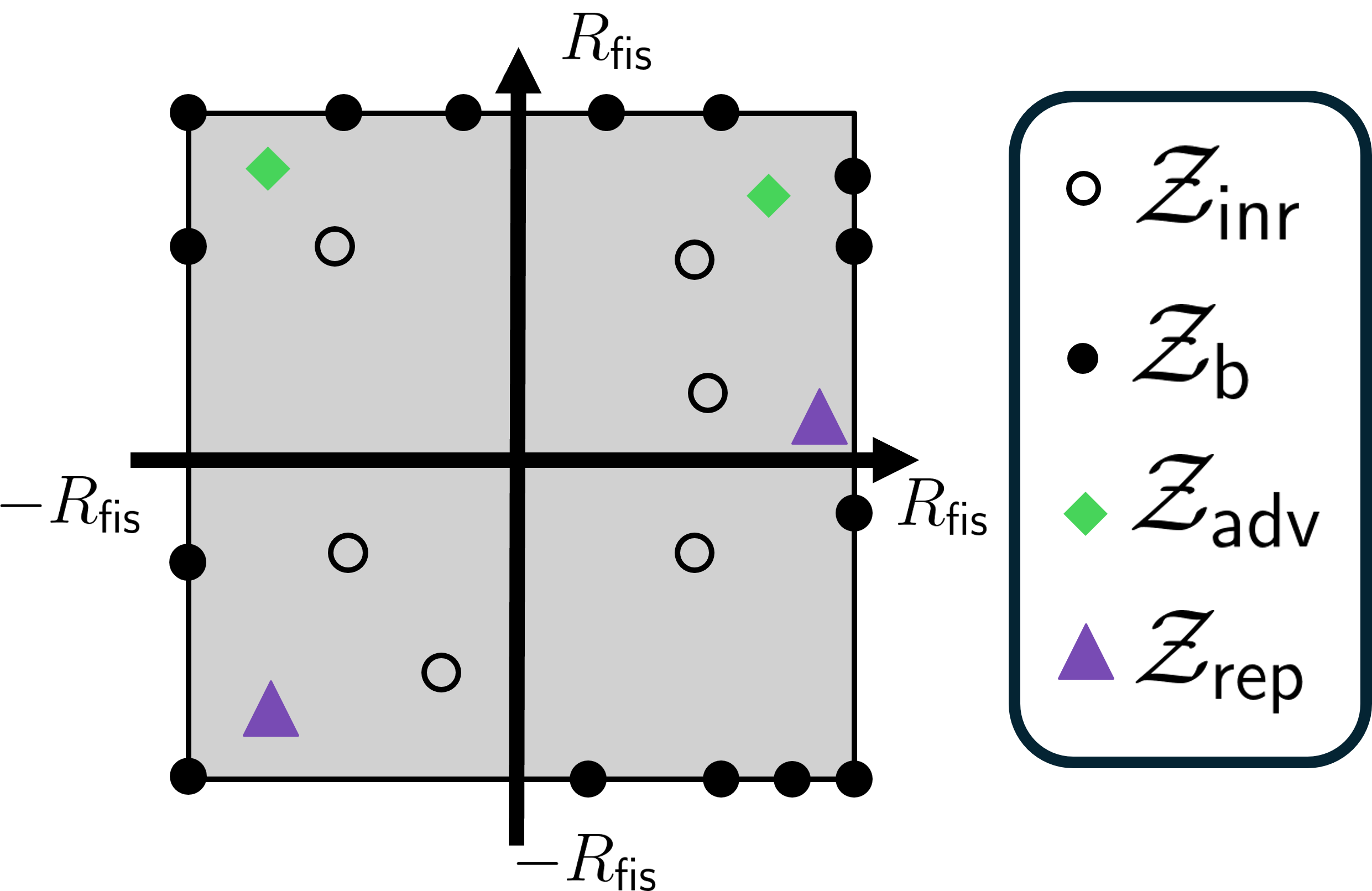}}
\caption{Training subsets: $\mathcal{Z}_{\mathsf{inr}}$, $\mathcal{Z}_{\mathsf{b}}$, $\mathcal{Z}_{\mathsf{adv}}$, and $\mathcal{Z}_{\mathsf{rep}}$.}
\label{fig:training_samples}
\end{figure}

The first training subset of vectors are sampled inside of $\mathcal{I}_{\mathsf{d}, 1} = \mathcal{C}(R_{\mathsf{fis}})$, i.e. from $\{\vec{z} \in \mathcal{I}_{\mathsf{d}, 1} = \mathcal{C}(R_{\mathsf{fis}}) \mid \|\vec{z}\|_{\infty} < R_{\mathsf{fis}}\}$. We randomly (as per uniform distribution) sample $I_{\mathsf{inr}} \in \mathbb{Z}_{>0}$ ($I_{\mathsf{inr}}$ is a pre-configured constant) such vectors. We call them \emph{interior training cases}, and respectively denote them as $\vec{z}_{\mathsf{inr}, i}$ ($i = 1$, $\dots$, $I_{\mathsf{inr}}$). We denote 
\begin{eqnarray}
\mathcal{Z}_{\mathsf{inr}} \sdef \{\vec{z}_{\mathsf{inr}, i}\}_{i=1, \dots, I_{\mathsf{inr}}}. \label{eqn:Z_i}
\end{eqnarray}

The second training subset of vectors are sampled from the boundaries of $\mathcal{I}_{\mathsf{d}, 1} = \mathcal{C}(R_{\mathsf{fis}})$, i.e. from $\{\vec{z} \in \mathcal{I}_{\mathsf{d}, 1} = \mathcal{C}(R_{\mathsf{fis}}) \mid \exists j \in \{1, \dots, n\}, \vec{z}[j] = R_{\mathsf{fis}}\}$. We randomly (as per uniform distribution) sample $I_{\mathsf{b}} \in \mathbb{Z}_{>0}$ ($I_{\mathsf{b}}$ is a pre-configured constant) such vectors. We call them \emph{boundary training cases}, and respectively denote them as $\vec{z}_{\mathsf{b}, i}$ ($i = 1$, $\dots$, $I_{\mathsf{b}}$). We denote 
\begin{eqnarray}
\mathcal{Z}_{\mathsf{b}} \sdef \{\vec{z}_{\mathsf{b}, i}\}_{i=1, \dots, I_{\mathsf{b}}}. \label{eqn:Z_b}
\end{eqnarray}

The third training subset of vectors are those found from $\mathcal{I}_{\mathsf{d}, 1} = \mathcal{C}(R_{\mathsf{fis}})$ by the well-known \emph{Projected Gradient Descent} (PGD) adversarial training/testing method ~\cite{madry2017towards}. The PGD adversarial training objective is set to find $I_{\mathsf{adv}} \in \mathbb{Z}_{>0}$ ($I_{\mathsf{adv}}$ is a pre-configured constant) vectors with maximized deviations. We call these vectors the \emph{adverse training cases}, and respectively denote them as $\vec{z}_{\mathsf{adv}, i}$ ($i = 1$, $\dots$, $I_{\mathsf{adv}}$). The specific process to find $\vec{z}_{\mathsf{adv}, i}$ is as follows. For each $i$, we randomly (as per uniform distribution) sample an initial value for $\vec{z}_{\mathsf{adv},i}$ from $\mathcal{I}_{\mathsf{d}, 1} = \mathcal{C}(R_{\mathsf{fis}})$, and then adjust $\vec{z}_{\mathsf{adv},i}$ as per PGD to maximize $d^{*}(\vec{z}_{\mathsf{adv}, i}, K, R_{\mathsf{fis}})$. We denote 
\begin{eqnarray}
\mathcal{Z}_{\mathsf{adv}} \sdef \{\vec{z}_{\mathsf{adv}, i}\}_{i=1, \dots, I_{\mathsf{adv}}}. \label{eqn:Z_adv}
\end{eqnarray}

The fourth training subset of vectors are sampled from known hard boundary training cases of previous epochs. The details are elaborated in Appendix~\ref{sect:ReplayBuffer}. Intuitively, we maintain a set $\mathcal{Q}_{\mathsf{rep}}$, aka \emph{replay buffer}, whose elements are boundary training cases with big deviations in previous epochs. At the beginning of the current epoch, we randomly sample $I_{\mathsf{rep}}' = \mathsf{min}\{ |\mathcal{Q}_{\mathsf{rep}}|, I_{\mathsf{rep}} \}$ vectors from $\mathcal{Q}_{\mathsf{rep}}$ (where $I_{\mathsf{rep}} \in \mathbb{Z}_{>0}$ is a pre-configured constant), respectively denoted as $\vec{z}_{\mathsf{rep}, i}$ ($i = 1$, $\dots$, $I_{\mathsf{rep}}'$). We denote 
\begin{eqnarray}
\mathcal{Z}_{\mathsf{rep}} \sdef \{\vec{z}_{\mathsf{rep}, i}\}_{i=1, \dots, I_{\mathsf{rep}}'}. \label{eqn:Z_rep}
\end{eqnarray}
\noindent At the end of the current epoch, we update $\mathcal{Q}_{\mathsf{rep}}$ with newly found hard boundary training cases (see Appendix~\ref{sect:ReplayBuffer} for details). $\hfill$\myitem\label{myitem:update_reply_buffer} 

$\mathcal{Z}_{\mathsf{inr}}$, $\mathcal{Z}_{\mathsf{b}}$,  $\mathcal{Z}_{\mathsf{adv}}$, and $\mathcal{Z}_{\mathsf{rep}}$ form the \emph{training set}, denoted as 
\begin{eqnarray}
\mathfrak{Z}
\sdef
\{\mathcal{Z}_{\mathsf{inr}},\mathcal{Z}_{\mathsf{b}}, \mathcal{Z}_{\mathsf{adv}},\mathcal{Z}_{\mathsf{rep}}\}. \label{eqn:per_epoch_training_set}
\end{eqnarray}

For each 
$\mathcal{Z}=\{\vec{z}_{1},\ldots,\vec{z}_{I}\} \in \mathfrak{Z}$, where $I = |\mathcal{Z}|$, 
sort $\mathcal{Z}$'s elements into $\vec{z}_{(1)},\dots,\vec{z}_{(I)}$,  so that
$d^{*}(\vec{z}_{(1)},K,R_{\mathsf{shr}})
\geqslant \dots \geqslant
d^{*}(\vec{z}_{(I)},K,R_{\mathsf{shr}})$ ($R_{\mathsf{shr}}$ is defined by  \eqref{eqn:R_shr}).
Given a pre-configured \emph{hard-case ratio} constant 
$\rho_{\mathsf{hrd}}\in(0,1]$, we define
\begin{eqnarray}
\mathcal{H}(\mathcal{Z}, \rho_{\mathsf{hrd}})
\sdef
\{\vec{z}_{(j)} \mid
1 \leqslant j \leqslant
\lceil \rho_{\mathsf{hrd}} I \rceil \}. \label{eqn:hard_case}
\end{eqnarray}
We then define the \emph{deviation loss} to be 
\begin{eqnarray}
L_{\mathsf{dev}}
\sdef
\sum_{\mathcal{Z} \in \mathfrak{Z}}
w_{\mathcal{Z}}
\left(
\frac{\sum\limits_{\vec{z}\in\mathcal{H}(\mathcal{Z}, \rho_{\mathsf{hrd}})}
\nu^{2}(\vec{z})}{|\mathcal{H}(\mathcal{Z}, \rho_{\mathsf{hrd}})|}
+
\frac{\lambda}{|\mathcal{Z}|}
\sum_{\vec{z}\in\mathcal{Z}}
\nu^{2}(\vec{z})
\right), 
\label{eqn:deviation_loss} \\
\mbox{where } \qquad\qquad
\nu(\vec{z}) \sdef d^{*}(\vec{z}, K, R_{\mathsf{shr}}), \qquad\qquad\qquad \label{eqn:v_z}
\end{eqnarray}
\noindent $w_{\mathcal{Z}}$ is the pre-configured loss weight constant assigned to the training subset $\mathcal{Z}$ ($\mathcal{Z} = \mathcal{Z}_{\mathsf{inr}}$, $\mathcal{Z}_{\mathsf{b}}$, $\mathcal{Z}_{\mathsf{adv}}$, or $\mathcal{Z}_{\mathsf{rep}}$),  
and $\lambda$ is another pre-configured loss weight constant. Empirically, we set $w_{\mathcal{Z}_{\mathsf{inr}}} = 30$, $w_{\mathcal{Z}_{\mathsf{b}}} = 120$, $w_{\mathcal{Z}_{\mathsf{adv}}} = 80$,  $w_{\mathcal{Z}_{\mathsf{rep}}} = 220$, and $\lambda = 0.15$,
The use of $\nu(\vec{z})$ (see \eqref{eqn:v_z}) in \eqref{eqn:deviation_loss} embodies the heuristics of (par\ref{myitem:heuristics_deviation_loss_function}).

Besides the deviation loss, Overall Procedure Step~\ref{step:train_inn} also suggests setting the normalization configuration parameters $\vec{\mu}$ to reflect the center of a preferred  FIS in $\mathbb{X}_{0}$. When the  FIS candidate in $\mathbb{X}_{1}$ is $\mathcal{I}_{\mathsf{d}, 1} = \mathcal{C}(R_{\mathsf{fis}})$, the corresponding  FIS candidate in $\mathbb{X}_{0}$ is $\mathcal{I}_{\mathsf{d}, 0} = \tilde{\mathcal{T}}^{-1}(\mathcal{I}_{\mathsf{d}, 1})$, where $\tilde{\mathcal{T}}^{-1}$ is decided by the current NINN under training. Therefore, we can define a \emph{centering loss} $L_{\mathsf{cen}}$ to encourage the NINN to move the center of $\mathcal{I}_{\mathsf{d}, 0}$ in $\mathbb{X}_{0}$ toward $\vec{\mu}$. Specifically, 
\begin{eqnarray}
L_{\mathsf{cen}}
\sdef 
\left\|
\tilde{\mathcal{T}}^{-1}(\vec{0})-\vec{\mu}
\right\|_{2}^{2}.
\label{eqn:center_loss}
\end{eqnarray}
\noindent where $\vec{0} \in \mathbb{R}^{n}$ is the zero vector, and $\| \vec{\bullet} \|_{2}$ is the Euclidean norm of $\vec{\bullet}$. 

Finally, we define the overall NINN training \emph{loss function} 
\begin{eqnarray}
L
\sdef
L_{\mathsf{dev}}
+
w_{\mathsf{cen}}L_{\mathsf{cen}}, 
\label{eqn:total_loss}
\end{eqnarray}
where $w_{\mathsf{cen}}$ is a pre-configured loss weight constant. Empirically, we set $w_{\mathsf{cen}} = 5$.

Based on the above, we propose the NINN Training algorithm for Overall Procedure Step~\ref{step:train_inn} as Algorithm~\ref{alg:INNTrainingForFISConstruction}. 

\begin{algorithm}[t]
\caption{NINN Training for Overall Procedure Step~\ref{step:train_inn}}
%
%\vspace{-2ex}
%
\label{alg:INNTrainingForFISConstruction}
\setcounter{ctr}{0}
\begin{tabbing}
888\=88\=88\=88\=88\=88\=88\=88\=88\=88\=88\kill
\stp{} \> \textbf{Input}: controller $\pi_{0}$, $f_{\mathsf{cld},0}$, \phantom{xxxxxxxxxxx} /* see \eqref{eqn:x_k+1_original_state_space} */ \\
       \>\> $\vec{\mu}$, $\vec{\sigma}$, \phantom{xxx} /* see \eqref{eqn:normalization} and  Overall Procedure Step~\ref{step:train_inn} */ \\
       \>\> $R_{\mathsf{fis}}$, $K$, $\rho_{\mathsf{shr}}$, $\rho_{\mathsf{hrd}}$, \phantom{xxl} /* see \eqref{eqn:intended_latent_fis_box}, \eqref{eqn:K_step_maximum_deviation_from_mathcal_C_R},  \eqref{eqn:R_shr}, \eqref{eqn:hard_case} */ \\
       \>\> $I_{\mathsf{inr}}$, $I_{\mathsf{b}}$, $I_{\mathsf{adv}}$, $I_{\mathsf{rep}}$, \phantom{xxx} /* see \eqref{eqn:Z_i}, \eqref{eqn:Z_b}, \eqref{eqn:Z_adv}, \eqref{eqn:Z_rep} */ \\
       \>\> total number of training epochs $E$, \\
       \>\>  early-stop test triggering threshold $\theta_{\mathsf{stp}}$, \\
       \>\> early-stop test set size $I_{\mathsf{stp}}$; \\
\stp{} \> \textbf{Output}: the trained NINN $(\tilde{\mathcal{T}}, \tilde{\mathcal{T}}^{-1})$; \\
\stp{} \> Initialize NINN $(\tilde{\mathcal{T}}, \tilde{\mathcal{T}}^{-1})$ with $\vec{\mu}$ and $\vec{\sigma}$; \\
\stp{} \> $\mathcal{I}_{\mathsf{d}, 1}$ := $\mathcal{C}(R_{\mathsf{fis}})$; // hyper-rectangular FIS candidate in $\mathbb{X}_{1}$.\\ 
\stp{} \> Initialize replay buffer $\mathcal{Q}_{\mathsf{rep}}$ := $\varnothing$; \\
\stp{step:init_c_cge} \> Initialize continuous good epoch count $c_{\mathsf{cge}}$ := $0$; \\
\stp{} \> \textbf{for} epoch := $1$ to $E$ \{ \phantom{xxxxxxxxxxxxxx} // epoch begins \\
\stp{} \>\> Sample $\mathcal{Z}_{\mathsf{inr}}$,   $\mathcal{Z}_{\mathsf{b}}$, $\mathcal{Z}_{\mathsf{adv}}$ from $\mathcal{I}_{\mathsf{d}, 1}$ as per \eqref{eqn:Z_i},  \eqref{eqn:Z_b}, \\
       \>\> and \eqref{eqn:Z_adv} respectively; \\
       \>\> // note \eqref{eqn:Z_i}, \eqref{eqn:Z_b}, \eqref{eqn:Z_adv} use $I_{\mathsf{inr}}$, $I_{\mathsf{b}}$, $I_{\mathsf{adv}}$. \\
\stp{} \>\> Sample $\mathcal{Z}_{\mathsf{rep}}$ from $\mathcal{Q}_{\mathsf{rep}}$ as per \eqref{eqn:Z_rep}; \phantom{l}// \eqref{eqn:Z_rep} uses $I_{\mathsf{rep}}$. \\
\stp{} \>\> // now training set $\mathfrak{Z}$ of \eqref{eqn:per_epoch_training_set} is ready. \\
\stp{} \>\> Compute $L_{\mathsf{dev}}$ as per \eqref{eqn:deviation_loss}; \\
       \>\> // \eqref{eqn:deviation_loss} uses $\pi_{0}$, $f_{\mathsf{cld}, 0}$, $R_{\mathsf{fis}}$, $K$, $\rho_{\mathsf{shr}}$, $\rho_{\mathsf{hrd}}$, $\tilde{\mathcal{T}}$, $\tilde{\mathcal{T}}^{-1}$ \\
       \>\> // via \eqref{eqn:x_k+1_original_state_space}, \eqref{eqn:x_k+1_latent_state_space}, \eqref{eqn:intended_latent_fis_box}, \eqref{eqn:K_step_maximum_deviation_from_mathcal_C_R}, \eqref{eqn:R_shr}, \eqref{eqn:hard_case}, \eqref{eqn:v_z}. \\
\stp{} \>\> Compute $L_{\mathsf{cen}}$ as per \eqref{eqn:center_loss}; \phantom{xxxx} // \eqref{eqn:center_loss} uses $\tilde{\mathcal{T}}^{-1}$, $\vec{\mu}$. \\
\stp{} \>\> Compute loss function $L$ as per \eqref{eqn:total_loss}; \\
       \>\> // \eqref{eqn:total_loss} uses $L_{\mathsf{dev}}$, $L_{\mathsf{cen}}$. \\
\stp{step:backpropagation} \>\> Backpropagate from $L$ to update NINN $(\tilde{\mathcal{T}}, \tilde{\mathcal{T}}^{-1})$; \\
\stp{} \>\> Update $\mathcal{Q}_{\mathsf{rep}}$ as per (par\ref{myitem:update_reply_buffer}); \\
\stp{step:early_stop_begin} \>\> // early-stop attempt begins \\
\stp{step:good_epoch_test} \>\> \textbf{if} $\forall \mathcal{Z}\in\mathfrak{Z}$, $\forall \vec{z}\in\mathcal{Z}$, $d^{*}(\vec{z},K,R_{\mathsf{fis}}) = 0$ \{ \phantom{x} // see \eqref{eqn:K_step_maximum_deviation_from_mathcal_C_R} \\
\stp{} \>\>\> $c_{\mathsf{cge}}$ := $c_{\mathsf{cge}}+1$; \\
\stp{} \>\> \} \textbf{else} \{  $c_{\mathsf{cge}}$ := $0$;  \} \\
\stp{step:early_stop_triggering_test} \>\> \textbf{if} $c_{\mathsf{cge}} \geqslant \theta_{\mathsf{stp}}$ \{ \\
\stp{} \>\>\> Sample $\mathcal{Z}_{\mathsf{stp}}$ from $\mathcal{I}_{\mathsf{d}, 1}$ as per \eqref{eqn:Z_stp};  //\eqref{eqn:Z_stp} uses $I_{\mathsf{stp}}$. \\
\stp{step:early_stop_triggering_condition} \>\>\> \textbf{if} $\forall \vec{z}\in\mathcal{Z}_{\mathsf{stp}}$, $d^{*}(\vec{z},K,R_{\mathsf{fis}}) = 0$ \{ \phantom{xxxxx} // see \eqref{eqn:K_step_maximum_deviation_from_mathcal_C_R}  \\
\stp{step:early_stop} \>\>\>\> \textbf{return} $(\tilde{\mathcal{T}}, \tilde{\mathcal{T}}^{-1})$; \phantom{xxxxxxxxxxxxxxx} // early-stop \\
\stp{} \>\>\> \} \\
\stp{step:reset_c_cge} \>\>\> $c_{\mathsf{cge}}$ := $0$; \\
\stp{step:early_stop_end} \>\> \} \phantom{xxxxxxxxxxxxxxxxxxxxx} // early-stop attempt ends. \\
\stp{} \> \} \phantom{xxxxxxxxxxxxxxxxxxxxxxxxxxxxxxxxx} // epoch ends. \\
\stp{} \> \textbf{return} $(\tilde{\mathcal{T}}, \tilde{\mathcal{T}}^{-1})$;
\end{tabbing}
\end{algorithm}

Note to further reduce training cost, Algorithm~\ref{alg:INNTrainingForFISConstruction} line~\ref{step:init_c_cge}, \ref{step:early_stop_begin} $\sim$ \ref{step:early_stop_end} attempt to stop the training early when the current NINN is already good enough. 

Specifically, Algorithm~\ref{alg:INNTrainingForFISConstruction} maintains a \emph{continuous good epoch count} $c_{\mathsf{cge}}$ throughout the training epochs (see Algorithm~\ref{alg:INNTrainingForFISConstruction} line~\ref{step:init_c_cge}, \ref{step:early_stop_begin} $\sim$ \ref{step:early_stop_end}). An epoch is considered ``good'' iff the condition of Algorithm~\ref{alg:INNTrainingForFISConstruction} line~\ref{step:good_epoch_test} holds, i.e. iff after the backpropagation adjustment of the NINN (by Algorithm~\ref{alg:INNTrainingForFISConstruction} line~\ref{step:backpropagation}), all the training cases' deviations from $\mathcal{I}_{\mathsf{d}, 1} = \mathcal{C}(R_{\mathsf{fis}})$ become $0$. 

If $c_{\mathsf{cge}}$ reaches a pre-configured constant threshold $\theta_{\mathsf{stp}}$ (see Algorithm~\ref{alg:INNTrainingForFISConstruction} line~\ref{step:early_stop_triggering_test}), then a more aggressive test for early-stop is triggered. Specifically, similar to the creation of $\mathcal{Z}_{\mathsf{adv}}$ (see \eqref{eqn:Z_adv}), we use the PGD training/testing method~\cite{madry2017towards} again, but this time for testing. We set the PGD adversarial testing objective to find $I_{\mathsf{stp}} \in \mathbb{Z}_{>0}$ ($I_{\mathsf{stp}}$ is a pre-configured constant) vectors from $\mathcal{I}_{\mathsf{d}, 1} = \mathcal{C}(R_{\mathsf{fis}})$ with maximized deviations. We denote these vectors as $\vec{z}_{\mathsf{stp}, i}$ ($i = 1$, $\dots$, $I_{\mathsf{stp}}$). The specific process to find $\vec{z}_{\mathsf{stp}, i}$ is as follows. For each $i$, we randomly (as per uniform distribution) sample an initial value for $\vec{z}_{\mathsf{stp}, i}$ from $\mathcal{I}_{\mathsf{d}, 1} = \mathcal{C}(R_{\mathsf{fis}})$, and then adjust $\vec{z}_{\mathsf{stp}, i}$ as per PGD to maximize $d^{*}(\vec{z}_{\mathsf{stp}, i}, K, R_{\mathsf{fis}})$. We denote
\begin{eqnarray}
\mathcal{Z}_{\mathsf{stp}} \sdef \{\vec{z}_{\mathsf{stp}, i}\}_{i=1, \dots, I_{\mathsf{stp}}}. \label{eqn:Z_stp}
\end{eqnarray}
\noindent The early stop is triggered iff $\forall \vec{z}_{\mathsf{stp}, i}$ ($i = 1$, $\dots$, $I_{\mathsf{stp}}$), the deviation from $\mathcal{I}_{\mathsf{d}, 1} = \mathcal{C}(R_{\mathsf{fis}})$ is $0$ (see Algorithm~\ref{alg:INNTrainingForFISConstruction} line~\ref{step:early_stop_triggering_condition} and \ref{step:early_stop}). Otherwise, the attempt to early-stop fails, and $c_{\mathsf{cge}}$ is reset to $0$ (see Algorithm~\ref{alg:INNTrainingForFISConstruction} line~\ref{step:reset_c_cge}).

%=============================
\section{Evaluation}
\label{sect:Evaluation}

%-----------------------------
\subsection{Evaluation Testbed Set Up}
\label{subsect:EvaluationSetUp}

We carry out evaluations on three classic nonlinear control applications: pendulum, cart pole, and \emph{Planar Vertical Take-Off and Landing} (PVTOL) of aerial vehicles. 

The  plant model implementations (see $f_{\mathsf{pltd}}$ of \eqref{eqn:x_k+1}) of the pendulum and cart pole are transplanted from Gymnasium~\cite{towers2026nips}, a well established repository of plant models (control testbeds) for training/testing AI-controllers. The  plant model implementation of the PVTOL is transplanted from Yang et al.~\cite{yang2024icml}.

As per (par\ref{myitem:controller_training_uses_discretized_plant_model}), to train AI-controllers, we connect the  plant model implementations with classic reinforcement learning algorithm implementations from the well-known Stable Baselines3 (SB3)~\cite{stable-baselines3} repository. In this way, we trained 20, 20, and 5 AI-controllers respectively for the pendulum, cart pole, and PVTOL control applications. Each AI-controller is an NN, whose scale is summarized by $n$, $m$, and $\eta$, respectively the dimension of the state space (i.e. NN's input size), the dimension of the actuation space (i.e. the NN's output size), and the number of neurons inside the NN. Table~\ref{tab:summary_of_ai_controllers} summarizes these NNs. $\hfill$\myitem\label{myitem:ai_controllers_for_evaluation}

\vspace{0.5ex}
\begin{table}[t]
\caption{Summary of AI-controllers}
\vspace{-3ex}
\label{tab:summary_of_ai_controllers}
\begin{center}
\begin{tabular}{l|c|c|c|p{10ex}|p{18ex}}
\hline%
Application & $n$ & $m$ & $\eta$ & activation function & Number of AI-Controllers Trained \\
\hline
Pendulum & 2 & 1 & 128 & \multirow{3}{*}{SB3 default}& 20 \\
Cart pole & 4 & 1 & 128 & & 20 \\
PVTOL & 6 & 2 & 24 & & 5 \\
\hline
\end{tabular} 
\end{center}
%
%\vspace{0.5ex}
%
$n$: state space dimension, i.e. an AI-controller's NN input size; \\
$m$: actuation space dimension, i.e. an AI-controller's NN output size; \\
$\eta$: number of neurons inside of an AI-controller's NN.
\end{table} 
\vspace{0.5ex}

We then connect the trained AI-controllers with the plant model implementations to generate 20 pendulum, 20 cart pole, and 5 PVTOL AI-controlled systems. For each AI-controlled system, we carry out our proposed Overall Procedure (see (par\ref{myitem:overall_procedure})) to find the analytically proven FIS. The configuration parameters (see \eqref{eqn:normalization}, Algorithm~\ref{alg:INNTrainingForFISConstruction}, and Algorithm~\ref{alg:ReplayBufferMaintenance}) are set as per Table~\ref{tab:overall_procedure_configuration_parameters}. 
Besides, we set $T_{\mathsf{max}}$ (see (par\ref{myitem:alpha_beta_CROWN_function_notation})) big enough, so that the calls to $\alpha\_\beta\_\mathsf{CROWN}$ are never terminated due to reaching $T_{\mathsf{max}}$. $\hfill$\myitem\label{myitem:ai_controlled_systems_for_evaluation}

\vspace{0.5ex}
\begin{table}[t]
\centering
\caption{Configuration parameter settings}
\vspace{-1.5ex}
\label{tab:overall_procedure_configuration_parameters}
%
%\resizebox{\linewidth}{!}{
%
\begin{tabular}{l|c|c}
\hline
Application & $\vec{\mu}$ & $\vec{\sigma}$ \\
\hline 
Pendulum & (0, 0)$^{\mathsf{T}}$ & (0.3, 1.5)$^{\mathsf{T}}$ \\
Cart pole & (0, 0, 0, 0)$^{\mathsf{T}}$ & (0.35, 1.0, 0.1, 1.6)$^{\mathsf{T}}$ \\
PVTOL & (0, 0, 0, 0, 0, 0)$^{\mathsf{T}}$ & (1.0, 1.0, 0.8, 1.0, 1.0, 1.0)$^{\mathsf{T}}$ \\
\hline
\end{tabular} \\
\vspace{0.5ex}
$\bullet^{\mathsf{T}}$: transpose of the vector/matrix represented by $\bullet$.\\
\vspace{0.5ex}
\begin{tabular}{p{9ex}|p{1ex}|p{1ex}|p{4ex}|p{3ex}|p{5ex}|p{5ex}|p{5ex}|p{5ex}}
\hline
$\!\!$Application & $\!\!\!R_{\mathsf{fis}}$ & $\!\!K$ & $\rho_{\mathsf{shr}}$ & $\rho_{\mathsf{hrd}}$ & $I_{\mathsf{inr}}$ & $I_{\mathsf{b}}$ & $I_{\mathsf{adv}}$ & $I_{\mathsf{rep}}$ \\
\hline
Pendulum & $\!\!\!$1.0 & 3 & 0.995 & 0.25 & 4096 & 4096 & 2048 & 2048 \\
Cart pole & $\!\!\!$1.0 & 3 & 0.995 & 0.25 & 4096 & 4096 & 2048 & 2048 \\
PVTOL & $\!\!\!$1.0 & 3 & 0.995 & 0.25 & 65536 & 65536 & 32768 & 32768 \\
\hline
\end{tabular} \\
\vspace{0.9ex}
\begin{tabular}{l|c|c|c|c|c|c}
\hline
Application & $E$ & $\theta_{\mathsf{stp}}$ & $I_{\mathsf{stp}}$ & $I^{\mathsf{max}}_{\mathsf{rep}}$ & $J_{\mathsf{add}}$ & $E_{\mathsf{rep}}$ \\
\hline
Pendulum & 3000 & 3 & 4096 & 16384 & 512 & 25 \\
Cart pole & 50000 & 5 & 4096 & 16384 & 512 & 25 \\
PVTOL & 100000 & 5 & 16384 & 262144 & 8192 & 25 \\
\hline
\end{tabular}
%
%} 
%
%
\end{table}
\vspace{0.5ex}

Particularly, $\vec{\mu}$ and $\vec{\sigma}$ values are set by domain experts. Settings on other configuration parameters are best effort. How to optimize these settings is beyond the scope of this paper. Nonetheless, even without optimization, our proposed solution already outperforms the comparison baseline in the evaluations (see Section~\ref{subsect:EvaluationResults}).

As discussed in Section~\ref{sect:Introduction} and \ref{sect:RelatedWork} (particularly, (par\ref{myitem:why_yang_et_al_method_does_not_fit})), to our best knowledge, there is no existing work in the literature that exactly solves our focus problem. The closest state-of-the-art solution is that proposed by Yang et al.~\cite{yang2024icml}. But this solution is for jointly designing the AI-controller and the FIS. That is, as the search for FIS goes on, the AI-controller is modified. This is not our focus problem, where the AI-controller is given a priori (e.g. when the legal/regulatory authority needs to certify a vendor provided AI-controlled system). To provide a comparison baseline, we have to alter Yang et al.'s solution: ban the modifications to the AI-controller (i.e. it has to use the AI-controllers generated in (par\ref{myitem:ai_controllers_for_evaluation})) during the search for FIS. We call the altered Yang et al.'s solution as the \emph{baseline solution}. $\hfill$\myitem\label{myitem:baseline_solution}

%-----------------------------
\subsection{Evaluation Results}
\label{subsect:EvaluationResults}

We use both our proposed solution (i.e. the Overall Procedure, see (par\ref{myitem:overall_procedure})) and the baseline solution (i.e. the altered Yang et al.\cite{yang2024icml}'s solution, see (par\ref{myitem:baseline_solution})) to find  FISs for the 20 pendulum, 20 cart pole, and 5 PVTOL AI-controlled systems described in (par\ref{myitem:ai_controllers_for_evaluation}),  (par\ref{myitem:ai_controlled_systems_for_evaluation}). All experiments were performed on a hardware platform equipped with a 32-core AMD Ryzen-series CPU, 64 GB of RAM, and an NVIDIA GeForce RTX 4090 GPU with 24 GB of VRAM.

The success rates of our proposed solution and the baseline solution are compared by Table~\ref{tab:success_rate}. Based on Table~\ref{tab:success_rate}, our proposed solution found FISs for all the 20 pendulum, 20 cart pole, and 5 PVTOL AI-controlled systems. In contrast, the baseline solution cannot find any FIS for any of these AI-controlled systems.

\vspace{0.5ex}
\begin{table}[t]
\centering
\caption{Success rates of finding FIS }
\vspace{-1.5ex}
\label{tab:success_rate}
\begin{tabular}{l|c|c}
\hline
Application & Proposed Solution & Baseline Solution \\
\hline
Pendulum & 100\% (20/20) & 0\% (0/20) \\
Cart pole & 100\% (20/20) & 0\% (0/20) \\
PVTOL & 100\% (5/5) & 0\% (0/5) \\
\hline
\end{tabular}
\end{table}

\begin{figure}[t]
\centering
\includegraphics[width=0.9\linewidth]{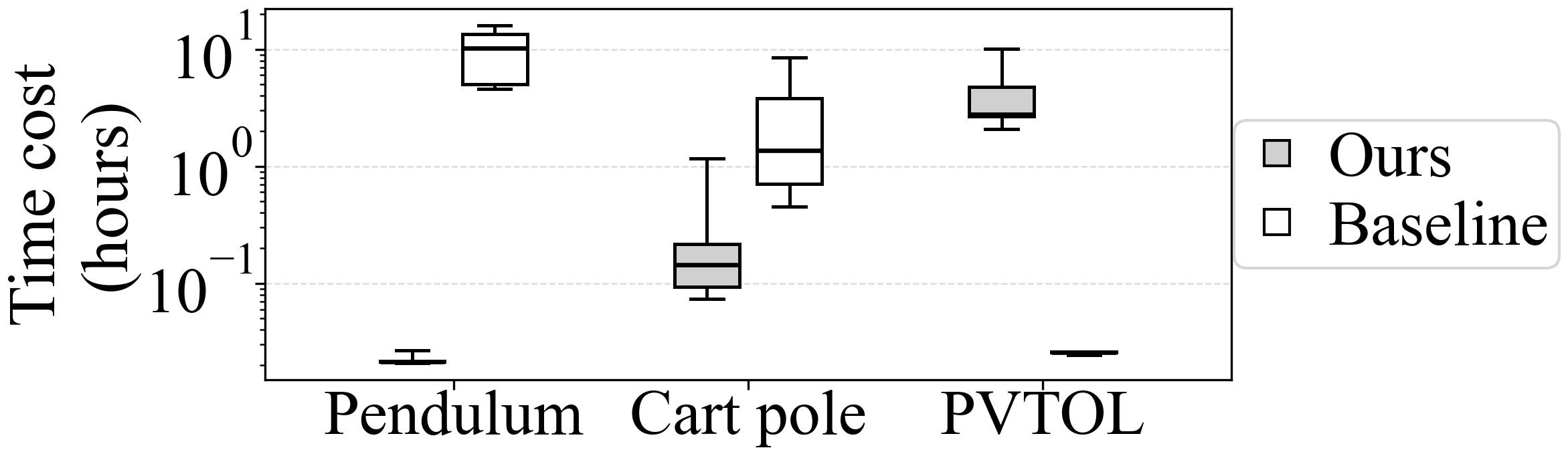}
\caption{Comparisons: time cost to run the FIS finding solution.}
\label{fig:time_costs}
\end{figure}

The time costs of our proposed solution and the baseline solution are compared by Fig.~\ref{fig:time_costs}. Here, each instance of time cost refers to the time cost to run the FIS finding solution for a given AI-controlled system. As per the figure, for pendulum and cart pole, our proposed solution's time costs are significantly smaller than the baseline solution's. While for PVTOL, the baseline solution's time costs are significantly smaller, partly due to its fast acknowledgments on failing to find the FISs (it founds 0 FIS, but our solution finds a FIS for each of the 5 PVTOL AI-controlled systems).

People are also interested in the sizes of the found FISs. As discussed in (par\ref{myitem:how_to_let_authority_trust}), the legal/regulatory authority can only trust the AI-controlled system when its initial state is within the found FIS. Therefore, the bigger the found FIS, the more freedom to choose the initial state. $\hfill$\myitem\label{myitem:why_fis_size_matters}

In the $n$-dimensional latent state space $\mathbb{X}_{1}$, the size of our found hyper-rectangular FIS $\mathcal{I}_{\mathsf{d}, 1}$ (referred to as ``$\mathbb{X}_{1}$ FIS'' in the following) is fixed to $(2R_{\mathsf{fis}})^n$. Hence the focus is the size of the FIS $\mathcal{I}_{\mathsf{d}, 0}$ in the $n$-dimensional original state space $\mathbb{X}_{0}$ (referred to as ``$\mathbb{X}_{0}$ FIS'' in the following). Table~\ref{tab:fis_size} lists the statistics on the sizes\footnote{Due to the nonlinearity of the AI-controlled system (see Fig.~\ref{fig:overall_framework}), a FIS $\mathcal{I}_{\mathsf{d},0}$ found by our solution in the original state space $\mathbb{X}_{0}$ is of highly irregular shape. To analytically calculate its size is hard. Instead, we use Monte Carlo~\cite{kalos2008monte} to estimate the size. We randomly (as per uniform distribution) sample vector $\vec{x}$ from a big enough hyper-rectangle $\mathcal{C}_{0}$ (where $\mathcal{I}_{\mathsf{d},0} \subseteq \mathcal{C}_{0}$), and measure the probability $p$ that $\vec{x} \in \mathcal{I}_{\mathsf{d},0}$ (in other words, $\tilde{\mathcal{T}}(\vec{x}) \in \mathcal{I}_{\mathsf{d},1} = \mathcal{C}(R_{\mathsf{fis}})$ in the latent state space $\mathbb{X}_{1}$). Then $\mathsf{size}(\mathcal{I}_{\mathsf{d},0})$ can be approximated by $p\,\mathsf{size}(\mathcal{C}_{0})$, where $\mathsf{size}(\bullet)$ represents the size of $\bullet$.} of the $\mathbb{X}_{0}$ FISs found by our solution. Note, the baseline solution finds no FIS at all, hence has no FIS size statistics. 
Also, as mentioned in Step~\ref{step:train_inn} of the Overall Procedure, $\vec{\mu}$ and $\vec{\sigma}$ (see Table~\ref{tab:overall_procedure_configuration_parameters} for their exact values set by domain experts) reflect the center and the dimension-wise radii of the preferred $\mathbb{X}_{0}$ FIS. Assume the preferred $\mathbb{X}_{0}$ FIS is hyper-rectangular, then its size is $\Pi_{i=1}^{n} (2\vec{\sigma}[i])$. 

Table~\ref{tab:fis_size} shows that for pendulum and cart pole, the $\mathbb{X}_{0}$ FIS sizes fluctuate from $3.7$\% to $150.4$\% of the preferred sizes. For PVTOL, the $\mathbb{X}_{0}$ FIS sizes are significantly smaller, only fluctuate from $0.016$\% to $0.025$\% of the preferred size. This could be due to the highly nonlinear nature of the PVTOL $\mathbb{X}_{0}$ FIS, as we assume the preferred $\mathbb{X}_{0}$ FIS is hyper-rectangular. 
\vspace{0.5ex}
\begin{table}[t]
\begin{center}
\caption{$\mathbb{X}_{0}$ FIS size statistics on our proposed solution}
\vspace{-1.5ex}
\label{tab:fis_size}
\begin{tabular}{p{10ex}|p{5ex}|p{7ex}|p{5ex}|p{5ex}|p{9ex}}
\hline
Application & min & median & mean & max & preferred \\
\hline 
Pendulum &0.929&1.307&1.288&1.308&1.800\\
Cart Pole &0.033&0.346&0.473&1.348&0.896\\
PVTOL &0.008&0.012&0.011&0.013&51.200\\
\hline
\end{tabular}
\end{center}
Note the baseline solution finds no FIS, hence has no FIS size statistics.
\end{table}
%
%\vspace{0.5ex}

The domain expert given preferred $\mathbb{X}_{0}$ FIS center $\vec{\mu}$ and preferred $\mathbb{X}_{0}$ FIS dimension-wise radii $\vec{\sigma}$ also motivate us to evaluate dimension-wise statistics of our found $\mathbb{X}_{0}$ FISs\footnote{Note the $\mathbb{X}_{1}$ FIS $\mathcal{I}_{\mathsf{d}, 1}$ is always a hypercube centered on the origin point of $\mathbb{X}_{1}$ with edge length $2R_{\mathsf{fis}}$.}. 
Specifically, given an $\mathbb{X}_{0}$ FIS $\mathcal{I}_{\mathsf{d},0}$, define its $i$th ($i = 1$, $\dots$, $n$) dimension upper and lower bounds respectively as
\begin{eqnarray}
\overline{\mathcal{I}}_{\mathsf{d},0}[i] \sdef \mathsf{max}_{\vec{x} \in \mathcal{I}_{\mathsf{d},0}} \{\vec{x}[i]\} \;\; \mbox{ and } \;\; 
\underline{\mathcal{I}}_{\mathsf{d},0}[i] \sdef \mathsf{min}_{\vec{x} \in \mathcal{I}_{\mathsf{d},0}} \{\vec{x}[i]\}. 
\nonumber
\end{eqnarray}
\noindent Then we define the $i$th dimension \emph{range size} of $\mathcal{I}_{\mathsf{d},0}$ as $(\overline{\mathcal{I}}_{\mathsf{d},0}[i] - \underline{\mathcal{I}}_{\mathsf{d},0}[i])$.
In contrast, we define the $i$th dimension \emph{preferred range size} as 
$2\vec{\sigma}[i]$.
Meanwhile, we define the $i$th dimension \emph{center drift} of $\mathcal{I}_{\mathsf{d},0}$ as $\left(\frac{1}{2}(\overline{\mathcal{I}}_{\mathsf{d},0}[i] + \underline{\mathcal{I}}_{\mathsf{d},0}[i]) - \vec{\mu}[i]\right)$. For an $\mathbb{X}_{0}$ FIS $\mathcal{I}_{\mathsf{d}, 0}$, its $\overline{\mathcal{I}}_{\mathsf{d}, 0}[i]$ and $\underline{\mathcal{I}}_{\mathsf{d}, 0}[i]$ can be approximated by randomly (as per uniform distribution) sampling the counterpart hyper-rectangular $\mathbb{X}_{1}$ FIS $\mathcal{I}_{\mathsf{d}, 1}$.

Table~\ref{tab:fis_dimension_wise_range_size_statistics} compares the dimension-wise range sizes of our found $\mathbb{X}_{0}$ FISs with their corresponding preferred range sizes. It shows that the range sizes on all dimensions fluctuate from $15.8$\% to $430.6$\% of the corresponding preferred range sizes. Note the baseline solution finds no FIS at all, hence has no dimension-wise range size statistics.

Table~\ref{tab:fis_dimension_wise_center_drift_statistics} lists the statistics on the dimension-wise center drifts of our found $\mathbb{X}_{0}$ FISs. To understand the relative extent of the center drifts, we can again compare them with the preferred range sizes listed in Table~\ref{tab:fis_dimension_wise_range_size_statistics}. Table~\ref{tab:fis_dimension_wise_center_drift_statistics} shows that the center drifts on all dimensions fluctuate from $-5.6$\% to $6.5$\% of the corresponding preferred range sizes (as noted in the title,  all numbers listed in Table~\ref{tab:fis_dimension_wise_center_drift_statistics} should $\times 10^{-3}$ before being used). The small center drifts may due to the use of the centering loss (see \eqref{eqn:center_loss}). They also mean that the $\mathbb{X}_{0}$ FISs' centers roughly meet the preferred centers. Note the baseline solution finds no FIS at all, hence has no dimension-wise center drift statistics.

\vspace{0.5ex}
\begin{table}[t]
\begin{center}
\caption{Dimension-wise range size statistics on $\mathbb{X}_{0}$ FISs found by our proposed solution}
\vspace{-1.5ex}
\label{tab:fis_dimension_wise_range_size_statistics}
%
%\vspace{0.5ex}
%
\begin{tabular}{p{10ex}|p{9ex}|p{4ex}|p{6ex}|p{4ex}|p{4ex}|p{10ex}}
\hline
Application & Dimension Index $i$ & min & median & mean & max & preferred range size \\
\hline 
\multirow{2}{*}{Pendulum} & 1 &0.822&0.822&0.823&0.840&0.600\\
& 2 &2.110&3.108&3.058&3.108&3.000\\
\hline
\multirow{4}{*}{Cart Pole} & 1 &1.000&1.913&2.018&3.014&0.700\\
& 2 &1.729&3.360&3.304&4.234&2.000\\
& 3 &0.381&0.632&0.626&0.847&0.200\\
& 4 &1.056&2.674&2.656&3.416&3.200\\
\hline
\multirow{6}{*}{PVTOL} & 1 &1.670&2.066&2.048&2.414&2.000\\
& 2 &2.239&2.384&2.382&2.609&2.000\\
& 3 &0.253&0.288&0.286&0.321&1.600\\
& 4 &1.886&2.284&2.269&2.693&2.000\\
& 5 &1.100&1.266&1.266&1.406&2.000\\
& 6 &0.948&1.022&1.013&1.061&2.000\\
\hline
\end{tabular} \\
\end{center}
Note the baseline solution finds no FIS, hence has no corresponding statistics.
\end{table}

%\vspace{0.5ex}
%
\begin{table}[t]
\begin{center}
\caption{Dimension-wise center drift (unit: $10^{-3}$) statistics on $\mathbb{X}_{0}$ FISs found by our proposed solution}
\vspace{-1.5ex}
\label{tab:fis_dimension_wise_center_drift_statistics}
%
%\vspace{0.5ex}
%
\begin{tabular}{p{10ex}|p{9ex}|p{7ex}|p{7ex}|p{7ex}|p{7ex}}
\hline
Application & Dimension Index $i$ & min & median & mean & max \\
\hline 
\multirow{2}{*}{Pendulum} & 1 &-6.557&-0.370&-0.679&-0.370\\
& 2 &-38.995&-2.024&-3.872&-2.024\\
\hline
\multirow{4}{*}{Cart Pole} & 1 &-29.255&1.053&-0.240&45.220\\
& 2 &-51.584&-4.826&-3.262&24.570\\
& 3 &-11.171&-0.255&-1.298&4.192\\
& 4 &-12.437&1.516&1.062&11.026\\
\hline
\multirow{6}{*}{PVTOL} & 1 &-45.552&-12.109&2.390&43.405\\
& 2 &-19.766&9.474&7.279&26.993\\
& 3 &-6.361&0.217&-2.090&0.999\\
& 4 &-46.040&14.341&7.399&70.504\\
& 5 &-65.811&-17.873&-30.718&-12.352\\
& 6 &-40.333&-15.338&-19.986&-2.152\\
\hline
\end{tabular} \\
\end{center}
Note the baseline solution finds no FIS, hence has no corresponding statistics.
\end{table}
%
%\vspace{0.5ex}

%-----------------------------------------
\subsection{Discussions on Threats to Validity}
\label{subsect:ThreatsToValidity}

\noindent \textit{1. Construct Validity Threats}

There can be other metrics to evaluate.

However, as per Section~\ref{sect:Introduction}, \ref{sect:ProblemFormulation}, particularly (par\ref{myitem:focus_of_this_paper}) and (par\ref{myitem:focus_problem}), finding FIS is the focus of this paper. Hence the success rate of finding FISs (see Table~\ref{tab:success_rate}) is undoubtedly the most relevant metric. 
Related to success rate, the computing time cost (see Fig.~\ref{fig:time_costs}) is also relevant: if the time cost is too big, the solution is impractical. 
The size of the FIS (see Table~\ref{tab:fis_size}) and the dimension-wise range size (see Table~\ref{tab:fis_dimension_wise_range_size_statistics}) are relevant because they reflect the freedom on the choice of the initial state (see (par\ref{myitem:why_fis_size_matters})). 
The dimension-wise center drift (see Table~\ref{tab:fis_dimension_wise_center_drift_statistics}) is relevant as it reflects how well the centers of the found FISs match the preferences set by domain experts. 

The above said, these metrics will be more relevant if the maximum feasible FIS size is known; or even more, if the maximum feasible FIS size for the given forbidden set $\mathcal{F}$ (see  (par\ref{myitem:definition_set_of_concerned_unsafe_states})) is known. However, how to calculate these values is a nontrivial problem way beyond the focus of this paper. It deserves many future papers. In this paper, we are laying the indispensable foundation for these future papers: to have a working solution that can find FIS in the first place.

\noindent \textit{2. Internal Validity Threats}

We carefully checked our code. 

For the transplantation of plant implementations from  Gymnasium~\cite{towers2026nips} and Yang et al.~\cite{yang2024icml}, we implemented simulation-based consistency checks to verify that, given the same input, the transplanted implementations output the same results as the original implementations. For altering Yang et al.~\cite{yang2024icml}'s solution to realize the baseline solution, we did not modify Yang et al.'s source code; instead, we only altered the configuration parameter to fix the controller. Two domain experts independently reviewed the baseline solution to check the correctness of the alteration.

Moreover, to reduce the risk of implementation bugs, we reused well-tested public code repositories as much as possible, such as SB3~\cite{stable-baselines3} and $\alpha,\beta$-CROWN~\cite{zhang2018nips}\cite{wang2021betanips}\cite{xu2020fast}\cite{xu2020nips}. These repositories have been widely used and actively maintained by the community.

Also, due to limited computing budget, we only evaluated 5 (instead of 20) PVTOL testbeds. However, the per application comparisons are already revealing important facts. 

\noindent \textit{3. External Validity Threats:}

\noindent \textit{3.1 Control Application Testbed Representativeness}

We can always expand our testbeds to cover more control applications. However, we claim the pendulum, cart pole, and PVTOL testbeds used in this paper are reasonably representative. They have long been used as classic testbeds in control-theory research/textbooks~\cite{fantoni2012non}\cite{sutton1998reinforcement}\cite{wu2023}\cite{yang2024icml}. In particular, pendulum and cart pole are among the most popular classic control testbeds in Gymnasium~\cite{towers2026nips}, while PVTOL is a classic high dimensional control testbed~\cite{yang2024icml}. Moreover, these three testbeds cover different state space dimensionalities, ranging from 2 to 6. In comparison, in Gymnasium, a well recognized control testbed repository, under the category of ``classic control,'' the testbeds' state space dimensionalities also range from 2 to 6, which matches our range of 2 to 6.

\noindent \textit{3.2 SB3 Representativeness}

SB3 is a well established open source repository for reinforcement learning algorithm implementations. It has over 13,500 stars and over 2,200 forks on GitHub~\cite{sb3github}, and over 1 million per month downloads as per the \texttt{ecosyste.ms} ecosystem list report~\cite{sb3ecosystems}.

\begin{comment}
    {\it 3.1 Human Oracle Representativeness}

To our best knowledge, Zolfagharian et al.'s STARLA~\cite{zolfagharian2023search} paper provides the most rigorous definition in the literature on human oracle for testing RL software. The use of the Gymnasium repository~\cite{towers_gymnasium_2023}\cite{Gymnasium} is also a widely-adopted practice to evaluate the RL software implementations.

\noindent {\it 3.2 RL Software Implementations' Representativeness}

We can always expand our benchmark to include more RL software implementations. However, we claim the SB3 implementations of A2C, PPO, and TD3 are reasonably representative, because 1) they are representative algorithms covering the two main categories of RL algorithms: on-policy and off-policy~\cite{sutton2018reinforcement}; 2) PPO is strongly recommended by OpenAI~\cite{latypov2023comprehensive}\cite{openai2017ppo}; and 3) SB3 is a reasonably sized, well-known and well-maintained open-source library of RL software implementations~\cite{raffin2021stable}\cite{stable-baselines3-github-url}.
\end{comment}

%=============================
\section{Conclusion}
\label{sect:Conclusion}

Our proposed solution can find analytically proven FISs for given AI-controlled systems. Our evaluations show that our proposed solution can find FISs for all 45 AI-controlled system testbeds, while the baseline solution can find none. Furthermore, our solutions's time costs are significantly smaller than those of the baseline solutions for 40 of the 45 testbeds. The centers of our found FISs also roughly meet the preferences set by domain experts. 

%=============================
%\section*{Acknowledgment}

% The research project related to this paper is supported
% in part by the Hong Kong RGC Theme-based Research Scheme T22-505/19-N (P0031331, RBCR, P0031259, RBCP);
% in part by RGC GRF 15207324 (P0051926, B-QCFM);
% in part by RGC Germany/HK Joint Research Scheme G-PolyU503/16;
% in part by HKSAR Government and HKJCCT P0041424 (ZB5A);
% and in part by the Hong Kong Polytechnic Univ.
% P0042701 (CE09),
% P0046487 (CE0F),
% P0047916 (TACW),
% P0042699 (CE55),
% P0045578 (CE1C),
% P0043884 (CD6R),
% P0047965 (TAEB),
% P0047964 (TAEA),
% P0033695 (ZVRD),
% P0013879 (BBWH),
% P0036469 (CDA8),
% P0043634 (1-TAB2),
% P0043647 (1-TABF),
% P0042721 (1-ZVG0),
% P0051205, P0051074, J2024A005,
% LTG22-25/IICA/33 (TDG 2022-25),
% and TDG22-25/SMS-11.

% Any opinions, findings, and conclusions or recommendations expressed in this publication are those of the authors and do not necessarily reflect the views of sponsors.

% The authors thank anonymous reviewers for their efforts/advice on improving this paper.

%=============================
\appendices

%=============================
\section{Replay Buffer}
\label{sect:ReplayBuffer}

The replay buffer $\mathcal{Q}_{\mathsf{rep}}$ mentioned in Section~\ref{subsect:TrainingNinn} stores at most $I^{\mathsf{max}}_{\mathsf{rep}}$ elements. 
At any time instance, denote the $j$th element in $\mathcal{Q}_{\mathsf{rep}}$ ($j = 1$, $\dots$, $|\mathcal{Q}_{\mathsf{rep}}|$) as tuple $(\vec{z}_{j}, \tilde{\nu}_{j})$, where $\vec{z}_{j} \in \mathcal{C}(R_{\mathsf{fis}})$ and  $\tilde{\nu}_{j}$ is a \emph{past deviation} recorded in one of the past epochs, detailed later.

Suppose the current epoch is the $e$th ($e \in \mathbb{Z}_{>0}$) epoch of training.

At the beginning of the current epoch, we sample $\mathcal{Z}_{\mathsf{rep}}$ (see \eqref{eqn:Z_rep}) from $\mathcal{Q}_{\mathsf{rep}}$.
Specifically, we sample $I'_{\mathsf{rep}} = \mathsf{min}\{I_{\mathsf{rep}}, |\mathcal{Q}_{\mathsf{rep}}|\}$ tuples from $\mathcal{Q}_{\mathsf{rep}}$ as per the following probabilistic distribution: for each  $(\vec{z}_{j}, \tilde{\nu}_{j}) \in \mathcal{Q}_{\mathsf{rep}}$, its probability to be sampled is
\begin{eqnarray}
\mathsf{Pr}((\vec{z}_{j}, \tilde{\nu}_{j})) =
\frac{\tilde{\nu}_{j}}
{\sum_{\ell = 1, \dots, |\mathcal{Q}_{\mathsf{rep}}|} \tilde{\nu}_{\ell}}.
\nonumber 
\end{eqnarray}
Denote the sampled $I'_{\mathsf{rep}}$ tuples as $\{(
\vec{z}_{\mathsf{rep}, i},  \tilde{\nu}_{\mathsf{rep}, i}
)\}_{i=1,\ldots,I'_{\mathsf{rep}}}$, thus $\mathcal{Z}_{\mathsf{rep}} \sdef \{\vec{z}_{\mathsf{rep},i}\}_{i=1,\ldots,I'_{\mathsf{rep}}}$, i.e. \eqref{eqn:Z_rep}. Note if $\mathcal{Q}_{\mathsf{rep}} = \varnothing$, then $I'_{\mathsf{rep}} = 0$, hence $\mathcal{Z}_{\mathsf{rep}} = \varnothing$.

At the end of the current epoch, we update $\mathcal{Q}_{\mathsf{rep}}$ as per Algorithm~\ref{alg:ReplayBufferMaintenance}. Besides the aforementioned parameters, additional input parameters to this algorithm include the following. 
$\mathcal{Z}_{\mathsf{b}}$ (see \eqref{eqn:Z_b}) is the set of boundary training cases sampled in the current epoch. 
$J_{\mathsf{add}} \in \mathbb{Z}_{>0}$ is a pre-configured constant denoting the preferred number of boundary training cases from $\mathcal{Z}_{\mathsf{b}}$ to be added to $\mathcal{Q}_{\mathsf{rep}}$. $E_{\mathsf{rep}} \in \mathbb{Z}_{>0}$ is a pre-configured constant denoting the period to overhaul $\mathcal{Q}_{\mathsf{rep}}$ (see Algorithm~\ref{alg:ReplayBufferMaintenance} line~\ref{step:overhaul_begin} $\sim$ \ref{step:overhaul_end}).

\begin{algorithm}[thbp]
\caption{Replay Buffer Update}
\label{alg:ReplayBufferMaintenance}
\vspace{-1.5ex}
\setcounter{ctr}{0}
\begin{tabbing}
888\=88\=88\=88\=88\=88\=88\=88\=88\=88\=88\kill
\stp{} \> \textbf{Input}: $\mathcal{Q}_{\mathsf{rep}}$, $I^{\mathsf{max}}_{\mathsf{rep}}$, $e$, $\mathcal{Z}_{\mathsf{b}}$, $J_{\mathsf{add}}$, $E_{\mathsf{rep}}$; \\
\stp{} \> \textbf{Output}: the updated $\mathcal{Q}_{\mathsf{rep}}$; \\
\stp{} \> Sort all elements of $\mathcal{Z}_{\mathsf{b}}$ into $\vec{z}_{\mathsf{b},(1)},\ldots,\vec{z}_{\mathsf{b},(I_{\mathsf{b}})}$ so that \\
       \> $\nu(\vec{z}_{\mathsf{b},(1)}) \geqslant \dots \geqslant \nu(\vec{z}_{\mathsf{b},(I_{\mathsf{b}})})$; \phantom{x} // $\nu(\vec{\bullet})$ is defined by  \eqref{eqn:v_z} \\
\stp{} \> $J'_{\mathsf{add}}$ := $\mathsf{min}\{J_{\mathsf{add}}, I_{\mathsf{b}}\}$; \\
\stp{} \> $\mathcal{A}_{\mathsf{rep}}$ := $\{\vec{z}_{\mathsf{b},(j)}\mid 1\leqslant j\leqslant J'_{\mathsf{add}}\}$; \\
\stp{} \> $\mathcal{Q}'_{\mathsf{rep}}$ := $\mathcal{Q}_{\mathsf{rep}} \cup \{(\vec{z}, \nu(\vec{z}))\mid \vec{z}\in\mathcal{A}_{\mathsf{rep}}\}$; \\
\stp{} \> Sort all elements of $\mathcal{Q}'_{\mathsf{rep}}$ into $(\vec{z}\,'_{1}, \tilde{\nu}'_{1})$, $\dots$, \\
       \> $(\vec{z}\,'_{|\mathcal{Q}'_{\mathsf{rep}}|}, \tilde{\nu}'_{|\mathcal{Q}'_{\mathsf{rep}}|})$ so that $\tilde{\nu}'_{1} \geqslant \dots \geqslant \tilde{\nu}'_{|\mathcal{Q}'_{\mathsf{rep}}|}$; \\
\stp{} \> $\mathcal{Q}_{\mathsf{rep}}$ := $\{(\vec{z}\,'_{j}, \tilde{\nu}'_{j})) \mid 1 \leqslant j \leqslant \mathsf{min}\{I^{\mathsf{max}}_{\mathsf{rep}},|\mathcal{Q}'_{\mathsf{rep}}|\}\}$; \\
\stp{step:overhaul_begin} \> // overhaul begins \\
\stp{} \> \textbf{if} ($|\mathcal{Q}_{\mathsf{rep}}|>0$) and ($e \bmod E_{\mathsf{rep}}=0$) \{ \\
\stp{} \>\> \textbf{for} $j$ := $1$ to $|\mathcal{Q}_{\mathsf{rep}}|$ \{ \\
\stp{} \>\>\> Update the $j$th element $(\vec{z}_{j}, \tilde{\nu}_{j})$ of $\mathcal{Q}_{\mathsf{rep}}$ using  \\
       \>\>\> the current NINN, i.e. set $\tilde{\nu}_{j}$ := $\nu(\vec{z}_{j})$; \\
\stp{} \>\> \} \\
\stp{step:overhaul_end} \> \} \phantom{xxxxxxxxxxxxxxxxxxxxxxxxxxxxxx} // overhaul ends.
\vspace{-2ex}
\end{tabbing}
\end{algorithm}

%=============================

\bibliographystyle{IEEEtran}
\bibliography{refs}

\end{document}